%% file: root.tex
\documentclass[letterpaper, 10 pt, conference]{ieeeconf}  

\IEEEoverridecommandlockouts                              

\input{tex_packages.tex}
\input{tex_settings.tex}

\theoremstyle{plain}

\newtheorem{assumption}{Assumption}
\newtheorem{theorem}{Theorem}

\title{\LARGE \bf
SafeFlowMPC: Predictive and Safe Trajectory Planning for Robot Manipulators with Learning-based Policies
}

\author{Thies Oelerich$^{1}$, Gerald Ebmer$^{1}$, Christian Hartl-Nesic$^{1}$, Andreas Kugi$^{1, 2}$
\thanks{$^{1}$All authors are with the Automation and Control Institute (ACIN), TU
    Wien, Vienna, Austria,
    {\tt\small \{oelerich, ebmer, hartl, kugi\}@acin.tuwien.ac.at}}%
\thanks{$^{2}$Andreas Kugi is with the center for Vision, Automation \& Control,
    AIT Austrian Institute of Technology GmbH, Vienna, Austria
{\tt\small andreas.kugi@ait.ac.at}}%
}

\begin{document}

\maketitle

\global\csname @topnum\endcsname 0
\global\csname @botnum\endcsname 0

\thispagestyle{empty}
\pagestyle{empty}

\begin{abstract}
	The emerging integration of robots into everyday life brings several major challenges. Compared to classical industrial applications, more flexibility is needed in combination with real-time reactivity. Learning-based methods can train powerful policies based on demonstrated trajectories, such that the robot generalizes a task to similar situations. However, these black-box models lack interpretability and rigorous safety guarantees. Optimization-based methods provide these guarantees but lack the required flexibility and generalization capabilities. This work proposes SafeFlowMPC, a combination of flow matching and online optimization to combine the strengths of learning and optimization. This method guarantees safety at all times and is designed to meet the demands of real-time execution by using a suboptimal model-predictive control formulation. SafeFlowMPC achieves strong performance in three real-world experiments on a KUKA 7-DoF manipulator, namely two grasping experiment and a dynamic human-robot object handover experiment. A video of the experiments is available at \href{https://www.acin.tuwien.ac.at/en/42d6}{https://www.acin.tuwien.ac.at/en/42d6}. The code is available at \href{https://github.com/TU-Wien-ACIN-CDS/SafeFlowMPC}{https://github.com/TU-Wien-ACIN-CDS/SafeFlowMPC}.
\end{abstract}

\section{INTRODUCTION}

Trajectory planning for robot manipulators is inherently difficult due to multiple factors. Firstly, the kinematics are non-linear, which leads to multiple solutions for a given task, and have singularities and physical limits that must be considered~\cite{lynchModernRoboticsMechanics2017}. Secondly, the environment in which the robot acts in may change during operation, requiring online adaptions of the motion during the operation. This might happen due to a change of the task initiated by an operator~\cite{oelerichBoundPlannerConvexSetBasedApproach2025}, moving obstacles in the scene~\cite{khansari-zadehRealtimeAvoidanceFast2012, kiemelSafeReinforcementLearning2024}, or other actors in the environment~\cite{oelerichModelPredictiveTrajectory2024, khoramshahiDynamicalSystemApproach2019}.
Thirdly, many task objectives for robot motions are hard to encode in numerical objective functions. Hence, existing optimization algorithms are difficult to apply. Examples include human-robot object handovers~\cite{oelerichModelPredictiveTrajectory2024, kimLearningbasedDynamicRobottoHuman2025, ortenziObjectHandoversReview2021}, cleaning of sinks~\cite{ungerProSIPProbabilisticSurface2024}, and generalizability to different environments~\cite{kimOpenVLAOpenSourceVisionLanguageAction2025}.
Lastly, physical interaction with the world requires safety considerations for actors in the scene, other objects, and the robot\textquotesingle s hardware. The robot must not collide with obstacles to avoid physical damage to the environment and itself.

Many solutions exist to tackle motion planning in such challenging environments. Notably, global optimization-based and sampling-based planners~\cite{elbanhawiSamplingBasedRobotMotion2014, schulmanMotionPlanningSequential2014} exist to plan a trajectory for the entire task with constraints. These require a numerical objective and are not capable to adapt the robot motion in real time, which is necessary to react to changes in the environment. To improve the computational efficiency, finite-horizon planning, e.g., using model-predictive control strategies~\cite{oelerichBoundMPCCartesianPath2025, beckModelPredictiveTrajectory2024a} or sampling-based approaches~\cite{otteRRTXAsymptoticallyOptimal2016}, is employed. The improved computational efficiency comes at the cost of degraded optimality of the solution but enables reactive behavior in real-time to changes in the environments. However, these approaches need well-posed problem formulations, i.e., a numerical reward must be designed for the task, which is often complicated. An alternative are learning-based methods~\cite{jannerPlanningDiffusionFlexible2022, sahaEDMPEnsembleofcostsguidedDiffusion2024, figueiredoprudencioSurveyOfflineReinforcement2024}, which enable fast inference times to plan motions online. The design of numerical rewards can be avoided by learning from a dataset of demonstrations~\cite{ravichandarRecentAdvancesRobot2020, rossReductionImitationLearning2011}. These demonstrations encode the desired behavior in diverse scenarios and the learning agent learns generalized behavior in similar situations. The downside of these approaches is the lack of systematic safety considerations as the learning agent is often modeled as a black box.

Flow matching~\cite{lipmanFlowMatchingGuide2024} and diffusion models~\cite{jannerPlanningDiffusionFlexible2022} have recently shown promising results in robot trajectory planning. Instead of learning the distribution over desired trajectories, these methods learn probability paths from a source distribution to the target distribution. This involves an iterative procedure during inference to create trajectories of the target distribution. During these iterative steps, the intermediate trajectories can be adapted to bias the model towards a desired behavior, e.g., enforcing safety. Current approaches include solving an optimization problem~\cite{romerDiffusionPredictiveControl2025}, using control barrier functions~\cite{daiSafeFlowMatching2025}, and using cost guidance~\cite{sahaEDMPEnsembleofcostsguidedDiffusion2024}. We extend this work by tailoring it further toward robot manipulators. As non-convex optimization~\cite{romerDiffusionPredictiveControl2025} compromises real-time capabilitiy as it my exhibit significant variety in terms of computation time. Cost guidance with gradients~\cite{sahaEDMPEnsembleofcostsguidedDiffusion2024} is unreliable as it cannot guarantee constraint satisfaction and does not scale well to complex environments and models. Control barrier functions~\cite{daiSafeFlowMatching2025} are difficult to design and therefore limit the performance of the system.

In real-world applications safety is critical and needs to be systematically enforced. For example, safety filters~\cite{wabersichPredictiveSafetyFilter2021, wabersichDataDrivenSafetyFilters2023, gargLearningSafeControl2024} are employed to project a potentially unsafe input signal into a set of safe inputs. This is generally possible for any kind of learning policy but has the disadvantage that the system behavior changes, which alters the state distribution and, thus, deteriorates performance of the policy. Therefore, it is advantageous to incorporate the safety filtering more deeply with the agent. The work in~\cite{liuSafeReinforcementLearning2025a} includes the safety consideration directly in the learning process, but this approach does not scale well to predictive planning.
In~\cite{fisacBridgingHamiltonJacobiSafety2019a}, the authors use Hamilton-Jacobi reachability~\cite{bansalHamiltonJacobiReachabilityBrief2017} to improve safety.  A one-step QP-projection is used to create a safe reinforcement learning policy update in~\cite{sunFISARForwardInvariant2021}. Constrained reinforcement learning often considers the expected constraint violations~\cite{sunFISARForwardInvariant2021}, which is extended to  worst-case safety constraints in~\cite{yangWCSACWorstCaseSoft2021}. A learning objective in Lagrangian formulation is proposed in~\cite{stookeResponsiveSafetyReinforcement2020} to improve safe behavior, which is applied to learning from demonstrations in~\cite{liu2024offlinesaferl}. The authors in~\cite{berkenkampSafeModelbasedReinforcement2017} rely on Gaussian processes with a safe backup policy to ensure safety.
However, these methods struggle to scale to complicated systems like robot manipulators, or cannot ensure constraint satisfaction at all times, providing only probabilistic bounds. For many constraints like collision avoidance or physical manipulator limits, probabilistic bounds are not sufficient, and the policy cannot be safely employed on the robot.

This work focuses on safe learning from demonstrations for online trajectory planning for robot manipulators. In particular, we use an adapted flow-matching model to enforce safe behavior and present constraint design considerations to enforce safety at all times.
The contributions of this work are as follows:
\begin{itemize}
	\item A novel safe flow-matching procedure, called SafeFlowMPC, is developed, where the flow-matching model improves performance and a suboptimal real-time optimization solver enforces safety. This provides a deep integration of the learning agent and constraint enforcement.
	\item SafeFlowMPC guarantees safety at all times by enforcing a safe terminal constraint.
	\item Reactive motion planning is demonstrated on a real-world 7-DoF robot manipulator. Two challenging scenarios are considered: An object grasp with a dynamically changing grasping position, and a human-robot object handover.
\end{itemize}
By integrating a flow-matching model with a real-time optimization backend, our method combines the adaptability of learning-based planning with the safety guarantees of classical control methods. This addresses the challenges of safety, reactivity, and real-time applicability in dynamic environments.

\section{Formulation}
\label{sec:formulation}

\begin{figure}[t]
	\centering
	\def\svgwidth{0.8\linewidth}
	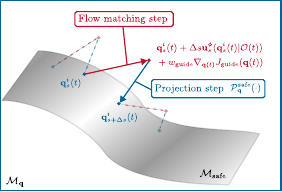
	\caption{Visual explanation of the optimization scheme of the proposed method.
	At time step $i$ the trajectory $\vec{q}^{i}_{s}(t)$ on the safety manifold $\mathcal{M}_{\mathrm{safe}}$ is improved by a flow matching step (red) and projected back onto $\mathcal{M}_{\mathrm{safe}}$ (blue) to obtain the safe trajectory $\vec{q}^{i}_{\mathrm{s + \Delta s}}(t)$. These two steps are executed multiple times for each time step (dashed lines) to traverse the flow from $s = 0$ to $s = 1$.}
	\label{fig:method_overview}
\end{figure}

The motion planning problem consists of finding a trajectory $\vec{q}^{*}(t)$ for the configuration state $\vec{q}$ of a robot manipulator that satisfies the motion constraints and achieves the desired objective. Formally, this is defined as
\begin{equation}
	\label{eq:motion_planning_problem}
	\vec{q}^{*}(t) = \arg\min_{\vec{q}(t) \in \mathcal{M}_{\mathrm{safe}}} \int_{t_{0}}^{t_{0}+T} J(\vec{q}(t)) \mathrm{d}t \text{~,}
\end{equation}
where the trajectory $\vec{q}^{*}(t)$ starts at $t_0$ with a trajectory duration $T$.
The manifold of allowed trajectories $\mathcal{M}_{\mathrm{safe}}$ is defined by enforcing all motion constraints, and the objective function $J(\vec{q}(t))$ is minimal when the desired objective is achieved.
A visualization of the formulation is shown in~\cref{fig:method_overview}.
For a simple movement task, the objective function is defined as the distance to a desired configuration. The constraints limit the allowed trajectories to the manifold $\mathcal{M}_{\mathrm{safe}}$ which contains all safe trajectories.
Difficulties arise due to the complicated form of $\mathcal{M}_{\mathrm{safe}}$ and the numerical objective function $J(\vec{q}(t))$ that achieves the desired task.
For the latter, learning from demonstrations~\cite{ravichandarRecentAdvancesRobot2020} has been successful in planning robot motions without explicitly specifying $J(\vec{q}(t))$. Instead, motions are learned that behave similar to a set of demonstrated trajectories. These learning-based methods often lack the ability to adhere to certain constraints and to keep the robot on $\mathcal{M}_{\mathrm{safe}}$ at all times.

In this work, we propose a combination of flow matching~\cite{lipmanFlowMatchingGuide2024} for learning behavior from demonstrations and optimizing $J(\vec{q}(t))$ with an optimization-based MPC to adhere to motion constraints and keep $\vec{q}(t)$ on $\mathcal{M}_{\mathrm{safe}}$.

\subsection{Safety Manifolds}
\label{ssec:manifolds}
The manifold of all possible trajectories is defined as $\mathcal{M}_{\vec{q}}$.
The safety manifold
\begin{equation}
	\mathcal{M}_{\mathrm{safe}} = \{\vec{q}(t): \vec{h}(\vec{q}(t)) \leq \vec{0} \land \vec{g}(\vec{q}(t)) = \vec{0}\}  \label{eq:safety_manifold}
\end{equation}
consists of all trajectories that satisfy the inequality constraints $\vec{h}(\vec{q}(t)) \leq \vec{0}$ and the equality constraints $\vec{g}(\vec{q}(t)) = \vec{0}$.
Furthermore, the performance manifold $\mathcal{M}_{\mathrm{*}} = \{\vec{q}(t): \vec{q}(t) \in \mathcal{M}_{\mathrm{safe}} \land J(\vec{q}(t)) \text{ is minimized}\}$ is constructed based on the trajectories that are in $\mathcal{M}_{\mathrm{safe}}$ and also minimize the objective function $J(\vec{q}(t))$.
When optimizing a trajectory $\vec{q}_{\mathrm{safe}}(t) \in \mathcal{M}_{\mathrm{safe}}$ to move it onto $\mathcal{M}_{\mathrm{*}}$, it is possible to define the projection operator
\begin{equation}
	\vec{q}_{*}(t) = \mathcal{P}_{\mathrm{safe}}^{*}(\vec{q}_{\mathrm{safe}}(t))
	\label{eq:projection_operator}
\end{equation}
that projects $\vec{q}_{\mathrm{safe}}(t)$ onto $\mathcal{M}_{\mathrm{*}}$. However, such a projection is generally non-trivial for manipulator motion planning as it poses a non-convex optimization problem. Another approach is to use the local properties of $\mathcal{M}_{\mathrm{safe}}$ at $\vec{q}_{\mathrm{safe}}(t)$ to move the trajectory iteratively toward $\mathcal{M}_{*}$.
This iterative procedure is formally defined as
\begin{equation}
	\vec{q}_{\mathrm{safe}}^{+}(t) = \mathcal{P}_{\vec{q}}^{\mathrm{safe}}(\vec{q}_{\mathrm{safe}}(t) + \Delta \vec{q}_{\mathrm{safe}}(t)) \text{~,}
	\label{eq:iterative_projection}
\end{equation}
where $\mathcal{P}_{\vec{q}}^{\mathrm{safe}}$ projects the initial trajectory modified by $\Delta \vec{q}_{\mathrm{safe}}(t)$ onto $\mathcal{M}_{\mathrm{safe}}$. The modification $\Delta \vec{q}_{\mathrm{safe}}(t)$ may be arbitrary, but it is assumed to be small. A possible choice is the gradient of the objective function $J(\vec{q}_{\mathrm{safe}}(t))$ projected onto the tangent space of $\mathcal{M}_{\mathrm{safe}}$ at $\vec{q}_{\mathrm{safe}}(t)$. This choice results in the projection-gradient method. Generally, any choice is viable as the projection operator $\mathcal{P}_{\vec{q}}^{\mathrm{safe}}$ projects $\vec{q}_{\mathrm{safe}}(t) + \Delta \vec{q}_{\mathrm{safe}}$ back onto $\mathcal{M}_{\mathrm{safe}}$.
The projection in~\cref{eq:iterative_projection} is simpler than the projection in~\cref{eq:projection_operator} due to the assumption that $\Delta \vec{q}_{\mathrm{safe}}$ is small. This is valid, if the projection is assumed to be performed by gradient-based non-convex optimization where the problem simplifies if the initial guess $\vec{q}_{\mathrm{safe}}(t)$ is close to the optimal solution $\vec{q}_{\mathrm{safe}}^{+}(t)$.
The components in~\cref{eq:iterative_projection} are described more thoroughly in Sections~\ref{ssec:flow_matching} and~\ref{ssec:trajectory_projection}.

\subsection{Flow Matching on Manifolds}
\label{ssec:flow_matching}

The trajectory modification $\Delta \vec{q}_{\mathrm{safe}}(t)$ in~\cref{eq:iterative_projection} is chosen to be computed by a flow matching model.
Flow matching models are used in motion planning to learn a motion policy from a set of demonstrations $\mathcal{D}$. This is achieved by learning a flow model that transforms samples from an initial distribution over trajectories $p_{0}(\vec{q}(t)|\mathcal{O}(t))$ into a target distribution $p_{1}(\vec{q}(t)|\mathcal{O}(t))$. It is assumed that the demonstrations in $\mathcal{D}$ are samples of $p_{1}(\vec{q}(t)|\mathcal{O}(t))$. The distributions are conditioned on the observation $\mathcal{O}(t)$, which differs from formulations used in similar work. Further, the source distribution $p_{0}(\vec{q}(t)|\mathcal{O}(t))$ may be an arbitrary distribution and is not limited to normal distributions. Sampling from a normal distribution over joint configurations will generally not output trajectories that are on the manifold $\mathcal{M}_{\mathrm{safe}}$.
This work focuses on having safe trajectories during the transfer from the source distribution $p_{0}(\vec{q}(t)|\mathcal{O}(t))$ to the target distribution $p_{1}(\vec{q}(t)|\mathcal{O}(t))$. The source distribution is the distribution over safe trajectories on $\mathcal{M}_{\mathrm{safe}}$ given the current observation $\mathcal{O}(t)$, and the target distribution is the distribution over safe trajectories that also optimize the objective function $J(\vec{q}(t))$ and, thus, lie on $\mathcal{M}_{*}$.

\subsubsection{Mathematical Formulation}
Transforming a sample $\vec{q}_{0}(t)$ from the initial distribution $p_{0}(\vec{q}(t)|\mathcal{O}(t))$ to a sample $\vec{q}_{1}(t)$ of $p_{1}(\vec{q}(t)|\mathcal{O}(t))$ is done using a conditional flow $\vec{u}^{\phi}_{s}(\vec{q}(t)|\mathcal{O}(t))$ with $0 \leq s \leq 1$ such that
\begin{align}
	\label{eq:flow_model}
	\vec{q}_{1}(t) & = \vec{q}_{0}(t) + \int_{0}^{1} \vec{u}^{\phi}_{s}(\vec{q}(t)|\mathcal{O}(t)) \mathrm{d} s \text{~.}
\end{align}
The flow $\vec{u}^{\phi}_{s}$ is parametrized by the parameters $\phi$, which are learned using the loss function
\begin{equation}
	\mathcal{L}_{\mathrm{FM}} = \mathbb{E}_{\substack{s, \vec{q}_{0} \sim p_{0}, \\ \vec{q}_{1} \sim p_{1}}} \lVert \vec{u}^{\phi}_{s}(\vec{q}(t)|\mathcal{O}(t)) - \vec{u}_{s}(\vec{q}(t)|\vec{q}_{0}(t), \vec{q}_{1}(t)) \rVert_{2}^{2} \text{~,}
	\label{eq:flow_loss}
\end{equation}
where the target flow
\begin{equation}
	\vec{u}_{s}(\vec{q}(t)|\vec{q}_{0}(t), \vec{q}_{1}(t)) = \frac{\partial}{\partial s}\mathcal{P}^{\mathrm{safe}}_{\vec{q}}\left(\vec{q}_{0}(t) + s  (\vec{q}_{1}(t)-\vec{q}_{0}(t))\right)
	\label{eq:target_flow}
\end{equation}
is computed using the projection operator from~\cref{eq:iterative_projection} and conditioned on $\vec{q}_{0}(t)$ and $\vec{q}_{1}(t)$ for its tractability.
The projection operator ensures that the path~\cref{eq:flow_model} always remains on $\mathcal{M}_{\mathrm{safe}}$ as long as $\vec{q}_{0}(t) \in \mathcal{M}_{\mathrm{safe}}$. It further ensures that the path ends on $\mathcal{M}_{*}$ as $s \rightarrow 1$ because $\vec{q}_{1}(t) = \mathcal{P}_{\vec{q}}^{\mathrm{safe}}(\vec{q}_{1}(t))$.
For more information on flow matching, the reader is referred to~\cite{lipmanFlowMatchingGuide2024}.

\subsubsection{Training Procedure}
\label{ssec:training_procedure}
Training our model is a two-step procedure. First a model is trained without any safety considerations by trying to match the demonstrations as done in standard flow matching~\cite{lipmanFlowMatchingGuide2024}. Secondly, this model is finetuned on a safety dataset, which is the original dataset, but each demonstration is adapted to be on $\mathcal{M}_{\mathrm{safe}}$. Specifically, the output distribution $p_{1}(\vec{q}(t)|\mathcal{O}(t))$ is computed by projecting each sample in $\mathcal{D}$ onto $\mathcal{M}_{\mathrm{safe}}$. Furthermore, the safe input distribution $p_{0}(\vec{q}(t)|\mathcal{O}(t))$, needed in~\cref{eq:target_flow}, is approximated by sampling and projection. The derivative of the projection operator~\cref{eq:iterative_projection}  with respect to $s$ in~\cref{eq:target_flow} is computed using finite differences.
This way the safety dataset consists of samples of~\cref{eq:target_flow} for different values of $s$ and samples of $\vec{q}_{0}(t)$ that move toward $\vec{q}_{1}(t)$.

\subsection{Trajectory projection}
\label{ssec:trajectory_projection}
Projecting joint-space trajectories $\vec{q}(t)$ onto the safety manifold $\mathcal{M}_{\mathrm{safe}}$ is described as a non-convex optimization problem to enforce the constraints $\vec{h}(\vec{q}(t))$ and $\vec{g}(\vec{q}(t))$.
This assumes that $\vec{h}(\vec{q}(t))$ and $\vec{g}(\vec{q}(t))$ are continuously differentiable to make gradient-based optimization feasible.
The projection operator $\mathcal{P}_{\vec{q}}^{\mathrm{safe}}$ in~\cref{eq:iterative_projection} is defined as
\begin{subequations}
	\label{eq:trajectory_projection}
	\begin{align}
		\mathcal{P}_{\vec{q}}^{\mathrm{safe}}(\vec{q}_{\mathrm{init}}(t)) = & \arg\min_{\vec{q}(t)} D(\vec{q}(t), \vec{q}_{\mathrm{init}}(t))       \\
		\text{s.t.}\quad                                                    & \label{eq:general_h} \vec{h}(\vec{q}(t)) = \begin{bmatrix}
			                                                                                                                 \vec{h}_{t}(\vec{q}(t)) \\
			                                                                                                                 \vec{h}_{T}(\vec{q}(t))
		                                                                                                                 \end{bmatrix} \leq \vec{0} \\
		                                                                    & \label{eq:general_g} \vec{g}(\vec{q}(t)) = \begin{bmatrix}
			                                                                                                                 \vec{g}_{t}(\vec{q}(t)) \\
			                                                                                                                 \vec{g}_{T}(\vec{q}(t))
		                                                                                                                 \end{bmatrix} = \vec{0}
	\end{align}
\end{subequations}
where $D(\cdot)$ is an appropriate distance measure for the trajectories. The constraints~\cref{eq:general_h} and~\cref{eq:general_g} comprise the constraints during the trajectory $\vec{h}_{t}(\vec{q}(t))$ and $\vec{g}_{t}(\vec{q}(t))$, and terminal constraints at the end of the trajectory $\vec{h}_{T}(\vec{q}(t))$ and $\vec{g}_{T}(\vec{q}(t))$ at time $t = t_0 + T$. These terminal constraints define the safety set $\mathcal{S}_{T}$.
\begin{assumption}
	\label{as:term_controller}
	A controller exists that has a control-invariant safety set that encompasses the terminal safety set $\mathcal{S}_{T}$ defined by $\vec{h}_{T}(\vec{q}(t))$ and $\vec{g}_{T}(\vec{q}(t))$.
\end{assumption}
\begin{theorem}
	\label{th:safe_traj}
	The projected trajectory $\vec{q}_{\mathrm{proj}}(t) = \mathcal{P}^{\mathrm{safe}}_{\vec{q}}(\vec{q}_{\mathrm{init}}(t))$ extending from the initial time $t = t_{0}$ to the end time $t = t_{0} + T$ is safe for all times $t > t_{0} + T$. Safety at time $t$ is defined by satisfying the constraints $\vec{h}(\vec{q}(t))$ and $\vec{g}(\vec{q}(t))$.
\end{theorem}
\begin{proof}
	The trajectory $\vec{q}_{\mathrm{proj}}(t)$ will end in the terminal safety set $\mathcal{S}_{T}$ defined by $\vec{h}_{T}(\vec{q}(t))$ and $\vec{g}_{T}(\vec{q}(t))$ at time $t = t_{0}+T$. At time $t = t_{0}+T$ the controller from~\cref{as:term_controller} is employed to keep the trajectory in the terminal safety set $\mathcal{S}_{T}$ for all times $t>t_{0}+T$.
\end{proof}

\subsection{Inference}

After training the model on the dataset $\mathcal{D}$ with the loss~\cref{eq:flow_loss}, the model is used for inference. Similar to the work~\cite{kimOpenVLAOpenSourceVisionLanguageAction2025, romerDiffusionPredictiveControl2025}, we use the model in closed-loop, where the model predicts the trajectory $\vec{q}^{i}(t)$ at time step $i$ for the time span $T = N T_{\mathrm{s}}$, with the prediction horizon $N$ and the sampling time $T_{\mathrm{s}}$. The trajectory $\vec{q}^{i}(t)$ is executed for the time $T_{\mathrm{s}}$ only, and then a new trajectory $\vec{q}^{i+1}(t)$ is computed at time step $i+1$. A step-by-step explanation is given in~\cref{alg:safe_flow_mpc}.
The prediction step is executed with
\begin{equation}
	\begin{aligned}
		\vec{q}^{i}_{s + \Delta s}(t) = \mathcal{P}^{\mathrm{safe}}_{\vec{q}}           \Bigl( & \vec{q}^{i}_{s}(t) + \Delta s \vec{u}_{s}^{\phi}(\vec{q}^{i}_{s}(t) | \mathcal{O}(t)) \\
		                                                                                       & + w_{\mathrm{guide}} \nabla_{\vec{q}(t)}J_{\mathrm{guide}}(\vec{q}(t))\Bigl)
	\end{aligned}
	\label{eq:flow_step}
\end{equation}
for a fixed flow step $\Delta s$ to get from $\vec{q}^{i}_{0}(t)$ to $\vec{q}^{i}_{1}(t)$, where the subscript denotes the probability flow time with bounds $0\leq s \leq 1$. This is visualized in~\cref{fig:method_overview}. The derivative of the guidance cost term $J_{\mathrm{guide}}(\vec{q}(t))$ is added with the weight $w_{\mathrm{guide}} > 0$, similar to diffusion guidance~\cite{sahaEDMPEnsembleofcostsguidedDiffusion2024}. Its particular design is application dependent and will be explained in the experiments section.
The trajectory for the next step $i+1$ is then defined by
\begin{equation}
	\vec{q}_{0}^{i+1}(t) =
	\begin{cases}
		\vec{q}_{1}^{i}(t) & t_{0} + T_{\mathrm{s}} < t < t_{0} + T \text{~,} \\
		\vec{q}^{i}_{T}(t) & t_{0} + T \leq t \leq t_{0} + T + T_{\mathrm{s}}
	\end{cases}
	\text{~,}
	\label{eq:next_traj}
\end{equation}
where $t_{0}$ is the starting time of $\vec{q}_{1}^{i}(t)$. This sets the initial trajectory for~\cref{eq:flow_step} at step $i+1$ to the optimized trajectory at step $i$ until it ends at $t_{0} + T$ and afterward applies the terminal trajectory $\vec{q}^{i}_{T}(t)$ according to \cref{as:term_controller}, which keeps the robot in the terminal safety set $\mathcal{S}_{T}$. A visual explanation of the iterative planning is provided in~\cref{fig:planning_overview}.
\begin{assumption}
	\label{as:safe_start}
	There exists a safe trajectory $\vec{q}^{0}(t)$ at the start of the robot movement such that the robot stays on the safety manifold $\mathcal{M}_{\mathrm{safe}}$.
\end{assumption}
This assumption is not very restrictive as it is often possible to stay in the safe initial state indefinitely.
\begin{theorem}
	\label{th:anytime_feasibility}
	The iterative trajectory generation~\cref{eq:next_traj} will keep the robot safe at all times.
\end{theorem}
\begin{proof}
	At the start of the movement, the robot is in a safe state according to~\cref{as:safe_start}. After moving from step $i$ to step $i+1$, the robot is safe because the projection in~\cref{eq:flow_step} keeps the new trajectory $\vec{q}^{i}_{s}(t)$ on the safety manifold $\mathcal{M}_{\mathrm{safe}}$. This proof by induction works as long as the projection in~\cref{eq:flow_step} converges. However, the projection operator $\mathcal{P}_{\vec{q}}^{\mathrm{safe}}(\cdot)$ introduced in~\cref{eq:trajectory_projection} is a non-convex optimization problem where convergence cannot be ensured. In case of a failure in~\cref{eq:flow_step} at $s = s_{\mathrm{fail}}$, the current trajectory $\vec{q}^{i}_{s_{\mathrm{fail}}}(t)$ is executed, which keeps the system safe according to~\cref{th:safe_traj} at all times due to enforcing the terminal safety set $\mathcal{S}_{T}$.

\end{proof}
\begin{algorithm}
	\footnotesize
	\caption{SafeFlowMPC Inference}
	\label{alg:safe_flow_mpc}
	\begin{algorithmic}[1]
		\Require Initial safe trajectory $\mathbf{q}_0^0(t)$, flow model $\mathbf{u}_s^\phi$, guidance cost $J_{\text{guide}}$, projection operator $\mathcal{P}_\mathbf{q}^\text{safe}$, weight for cost guidance $w_{\mathrm{guide}}$
		\Require Number of flow steps $N_s$, prediction horizon $N$, sampling time $T_s$

		\State Initialize $i \gets 0$, $t_0 \gets 0$
		\While{task not completed}
		\State $\mathbf{q}_s^i(t) \gets \mathbf{q}_0^i(t)$ \Comment{Initialize flow trajectory}
		\State $\Delta s \gets 1/N_s$ \Comment{Flow step size}
		\State Retrieve observation $\mathcal{O}(t)$

		\For{$k = 1$ to $N_s$} \Comment{Safe flow matching steps}
		\State Compute $\vec{q}^{i}_{s + \Delta s}(t)$ from $\vec{q}^{i}_{s}(t)$ using~\cref{eq:flow_step}
		\State $s \gets s + \Delta s$
		\EndFor

		\State $\mathbf{q}_1^i(t) \gets \mathbf{q}_{s}^i(t)$ \Comment{Final flow trajectory}

		\State Execute the trajectory $\mathbf{q}_1^i(t)$ for one time step $T_{\mathrm{s}}$

		\State Update the initial trajectory using~\cref{eq:next_traj}
		\State $i \gets i + 1$
		\State $t_0 \gets t_0 + T_s$
		\EndWhile

	\end{algorithmic}
\end{algorithm}

\begin{figure}[t]
	\centering
	\def\svgwidth{0.8\linewidth}
	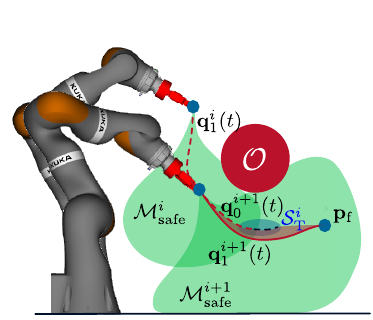
	\caption{Planning scheme of SafeFlowMPC for the trajectory planning of a robot manipulator around an obstacle $\mathcal{O}$. At step $i$, the robot plans the trajectory $\vec{q}_{1}^{i}(t)$ on the safety manifold $\mathcal{M}^{i}_{\mathrm{safe}}$. This trajectory ends in the terminal safety set $\mathcal{S}_{T}^{i}$ indicated by the blue shaded area. At step $i+1$ the planner reuses the previous trajectory according to~\cref{eq:next_traj} and transfers it (red shaded area) to $\vec{q}_{1}^{i+1}(t)$ on $\mathcal{M}^{i+1}_{\mathrm{safe}}$.}
	\label{fig:planning_overview}
\end{figure}

\section{Practical Implementation: Robot Manipulator}
\label{sec:implementation_manipulator}
This section discusses the implementation details of using the safe trajectory generation described in~\cref{sec:formulation} for the motion planning of a robot manipulator, for which
$\vec{q}(t)$ denotes the joint-position trajectory. For shorter notation, we use the state $\vec{x}^{\mathrm{T}}(t) = [ \vec{q}^{\mathrm{T}}(t), \dot{\vec{q}}^{\mathrm{T}}(t), \ddot{\vec{q}}^{\mathrm{T}}(t), \dddot{\vec{q}}^{\mathrm{T}}(t) ]$. The constraints of the safety manifold $\mathcal{M}_{\mathrm{safe}}$ in~\cref{eq:safety_manifold} are given by
\begin{equation}
	\vec{h}(\vec{q}(t)) = \begin{cases}
		\underline{\vec{x}} \leq \vec{x}(t) \leq \overline{\vec{x}}          & \forall \quad t_{0} \leq t \leq t_{0} + T                                  \\
		\vec{f}_{\mathrm{fk}, i}(\vec{q}(t)) \in \mathcal{S}_{\mathrm{free}} & \forall \begin{aligned} \quad t_{0} \leq t \leq t_{0} + T \\
			                                                                               \quad\quad i = 1, \ldots, N_{\mathrm{fk}}\end{aligned}
	\end{cases}
	\label{eq:robot_h}
\end{equation}
and
\begin{equation}
	\vec{g}(\vec{q}(t)) = \begin{cases}
		\vec{x}(t_{0}) = \vec{x}_{0}                                                                                                                 \\
		[\dot{\vec{q}}^{\mathrm{T}}(t_{0} + T), \ddot{\vec{q}}^{\mathrm{T}}(t_{0} + T), \dddot{\vec{q}}^{\mathrm{T}}(t_{0} + T)] = \vec{0} \text{~,} \\
	\end{cases}
	\label{eq:robot_g}
\end{equation}
where the initial state $\vec{x}_{0}$ is taken from the previous trajectory according to~\cref{eq:next_traj}, which ensures continuity across time steps. The state $\vec{x}$ is kinematically bounded by the lower bound $\underline{x}$ and the upper bound $\overline{x}$.
The second constraint in~\cref{eq:robot_g} ensures containment in the terminal set $\mathcal{S}_{T}$ and, thus, safety of the robot for $t \geq t_{0} + T$.
The constraint $\vec{f}_{\mathrm{fk}, i}(\vec{q}(t)) \in \mathcal{S}_{\mathrm{free}}$ in~\cref{eq:robot_h} uses the forward kinematics $\vec{f}_{\mathrm{fk}, i}(\vec{q}(t))$ of the manipulator to ensure that the kinematic chain is collision free. In this work, the collision formulation from~\cite{oelerichBoundPlannerConvexSetBasedApproach2025} is utilized, which uses collision-free convex sets around $N_{\mathrm{fk}}$ key points of the manipulator. This approach scales well with the number of obstacles and is real-time capable. For more information on this approach, the reader is referred to~\cite{oelerichBoundPlannerConvexSetBasedApproach2025}.
The forward kinematics $\vec{f}_{\mathrm{fk}, i}(\vec{q}(t))$ are nonlinear, rendering the problem~\cref{eq:trajectory_projection} non-convex. This implies that~\cref{eq:trajectory_projection} has multiple (local) minima, making it difficult to reliably solve. Employing a non-convex solver for this projection and solving it optimally is computationally infeasible as the projection is performed multiple times to get from $\vec{q}^{i}_{0}(t)$ to $\vec{q}^{i}_{1}(t)$ using~\cref{eq:flow_step}. Therefore, this work adapts a suboptimal approach using the real-time-iteration (RTI) scheme~\cite{diehlRealTimeIterationScheme2005} in the \emph{acados} framework~\cite{verschuerenAcadosModularOpensource2022}, which solves exactly one QP problem based on~\cref{eq:trajectory_projection} for each step~\cref{eq:flow_step}. In practice, a suboptimal solution is often sufficient and will converge to the optimal solution over multiple steps using~\cref{eq:flow_step}.
In this work, $N_{s} = 7$ steps of~\cref{eq:flow_step} are used in each time step.

\section{Experiments}
\label{sec:experiments}
Three experiments are presented in this section to evaluate performance, versatility, and adherence to safety constraints of the developed method. The first experiment utilizes our method to learn from a global trajectory planner to achieve fast local planning in a global context. The second experiment focuses on reactive online replanning in an obstructed environment. The third experiment is a dynamic human-robot object handover where the robot learns behavior from a human-human handover dataset~\cite{kimLearningbasedDynamicRobottoHuman2025}.

We compare our method to the following formulations:
\begin{enumerate}
	\item \textit{VP-STO}: VP-STO~\cite{jankowskiVPSTOViapointbasedStochastic2023} is a global trajectory planner based on stochastic sampling.
	\item \textit{BoundMPC}: BoundMPC~\cite{oelerichBoundMPCCartesianPath2025} is a model-predictive trajectory planner with a path-following-based formulation. The path is computed by BoundPlanner~\cite{oelerichBoundPlannerConvexSetBasedApproach2025}.
	\item \textit{BC}: Behavior cloning~\cite{pomerleauALVINNAutonomousLand1988} using a multi-layer perceptron, which takes the last trajectory and the current state as input and outputs the next trajectory. The training only includes safe trajectories to incentivize safe behavior.
	\item \textit{FM}: The baseline flow-matching training procedure without any safety considerations~\cite{lipmanFlowMatchingGuide2024}.
	\item \textit{Ours with NL-Opt}: Our proposed method, but the projection~\cref{eq:trajectory_projection} is solved to optimality instead of using the real-time iteration approach described in~\cref{sec:implementation_manipulator}. The IPOPT~\cite{wachterImplementationInteriorpointFilter2006} solver with the MA57 linear solver~\cite{duffMA57aCodeSolution2004} is used in this work. To increase the speed, only 4 steps of~\cref{eq:flow_step} are used in each time step.
	\item \textit{Ours w/o finetuning}: Our proposed method without the finetuning on the safety dataset described in~\cref{ssec:training_procedure}. The training simplifies to the standard flow matching training procedure.
\end{enumerate}

The flow network architecture for all flow matching models is a temporal U-Net network as in~\cite{jannerPlanningDiffusionFlexible2022}.
All online planners are executed at \SI{10}{\hertz}. A video of the experiments is available at \href{https://www.acin.tuwien.ac.at/en/42d6}{https://www.acin.tuwien.ac.at/en/42d6}.

\subsection{Experiment 1: Global trajectory planning made local}
\label{sec:global_local_planning}
By learning from global trajectory planners, SafeFlowMPC is able to exhibit close-to-optimal behavior while being real-time executable, thus, combining the advantages of global and local planners.
To learn this behavior, a dataset $\mathcal{D}$ was created using the global planner VP-STO~\cite{jankowskiVPSTOViapointbasedStochastic2023} in the environment depicted in~\cref{fig:scen_global_planner} with two obstacles forming a narrow passage. The goal is to plan from an initial joint configuration $\vec{q}_{\mathrm{init}}(0)$, where the end-effector is at pose $\vec{p}_{0}$, to a final end-effector pose $\vec{p}_{\mathrm{f}}$, where an object is picked up.
The dataset consists of 4100 trajectories, i.e., 4000 for training and 100 for evaluation. The guidance function
\begin{equation}
	J_{\mathrm{guide}}(\vec{q}(t)) = \sum_{j=1}^{N_{s}} D_{\mathrm{pose}}\left(\vec{f}_{\mathrm{fk, ee}}(\vec{q}(t_{0} + j T_{\mathrm{s}})), \vec{p}_{\mathrm{f}}\right)
	\label{eq:exp1_guidance}
\end{equation}
used in~\cref{eq:flow_step} is defined by the distance between the current end-effector pose $\vec{f}_{\mathrm{fk, ee}}(\vec{q}(t))$ and the final pose $\vec{p}_{\mathrm{f}}$, where the function $D_{\mathrm{pose}}$ measures the distance between two poses. It is the sum of the Cartesian distance for the position and the rotation vector angle difference between two orientations.
Equation~\cref{eq:exp1_guidance} ensures convergence to the final pose $\vec{p}_{\mathrm{f}}$.
The weight
\begin{equation}
	w_{\mathrm{guide}}(t) = \exp{\left(-\alpha \left(\frac{\lVert \vec{f}_{\mathrm{fk, ee}}(\vec{q}(t)) - \vec{p}_{0}\rVert_{2}}{\lVert \vec{p}_{0} - \vec{p}_{\mathrm{f}}\rVert_{2}} - \beta\right)\right)}
\end{equation}
is chosen such that it only influences the behavior close to $\vec{p}_{\mathrm{f}}$ with the parameters $\alpha>0$ and $\beta>0$.
The observation $\mathcal{O}(t)$ consists of the joint positions of the previous $10$ time steps, the Cartesian positions of the collision-checked points on the robot, the current pose of the end effector, and the desired pose of the end effector $\vec{p}_{\mathrm{f}}$.

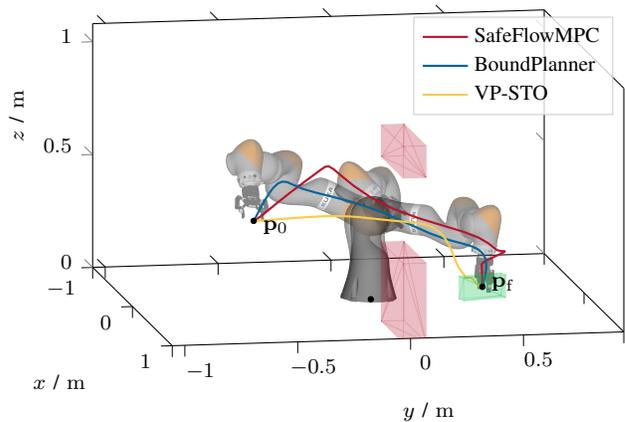
\begin{figure}
	\centering
	\addtolength\abovecaptionskip{-15pt}
	\def\axisdefaultwidth{\linewidth}
	\def\axisdefaultheight{0.7\linewidth}
	\begin{tikzpicture}
		\definecolor{darkgray176}{RGB}{176,176,176}
		\definecolor{acinblue}{RGB}{0,102,153}
		\definecolor{acinyellow}{RGB}{252, 204, 71}
		\definecolor{acingreen}{RGB}{0,190,65}
		\definecolor{acinred}{RGB}{186,18,43}
		\definecolor{lightgray204}{RGB}{204,204,204}

		\begin{axis}[
				view={80}{10},
				xlabel={$x$ / \si{\meter}},
				ylabel={$y$ / \si{\meter}},
				zlabel={$z$ / \si{\meter}},
				zmin=0,
				xmin=-1,
				xmax=1,
				axis equal,
				tick align=outside,
				tick pos=left,
				x grid style={darkgray176},
				y grid style={darkgray176},
				xtick style={color=black},
				ytick style={color=black},
				legend cell align={left},
				legend style={
						fill opacity=0.8,
						draw opacity=1,
						text opacity=1,
						draw=lightgray204,
					},
				legend entries ={SafeFlowMPC, BoundPlanner, VP-STO}
			]

			\node[inner sep=0pt, opacity=0.5] (qf) at (axis cs:0, 0.135, 0.36)
			{\includegraphics[width=.415\textwidth]{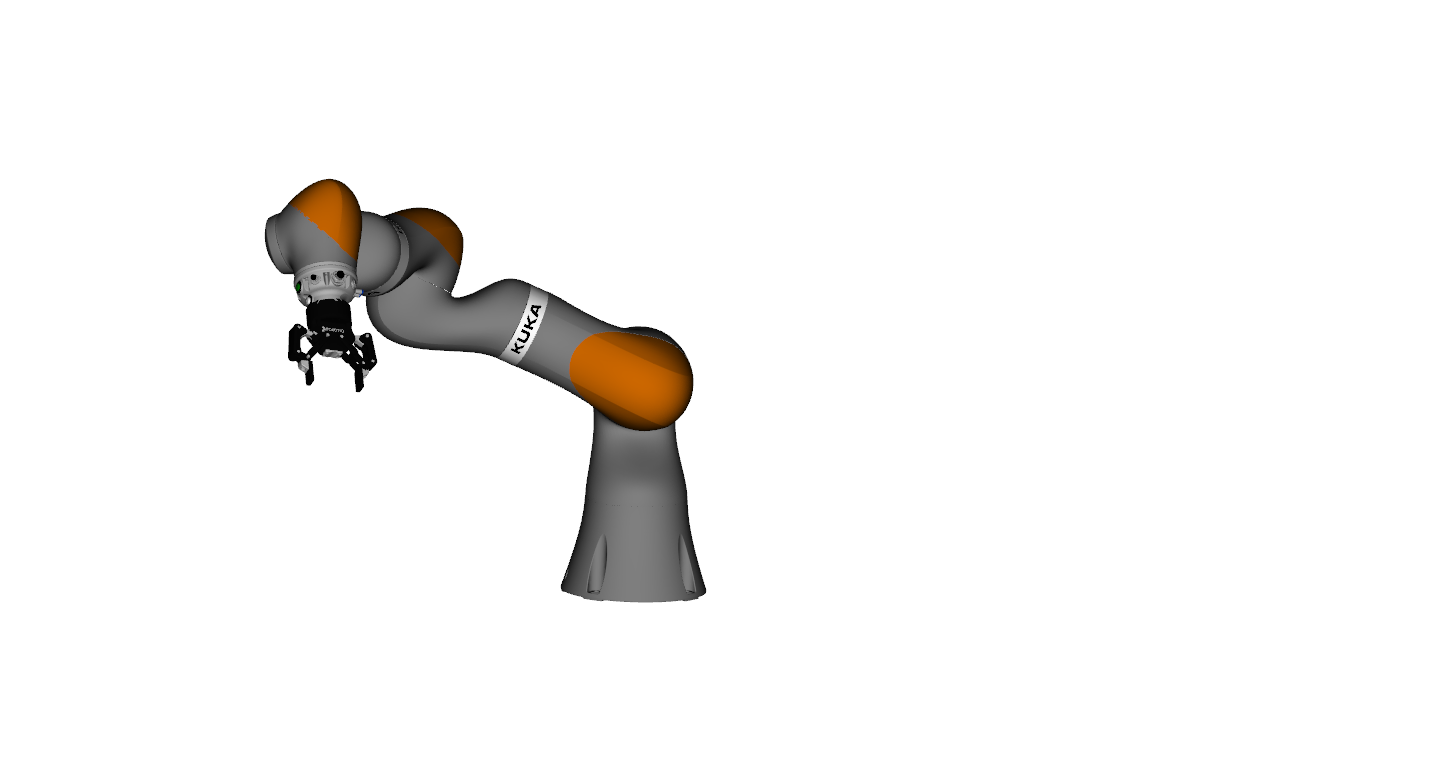}};
			\node[inner sep=0pt, opacity=0.5] (qf) at (axis cs:0, 0.135, 0.36)
			{\includegraphics[width=.415\textwidth]{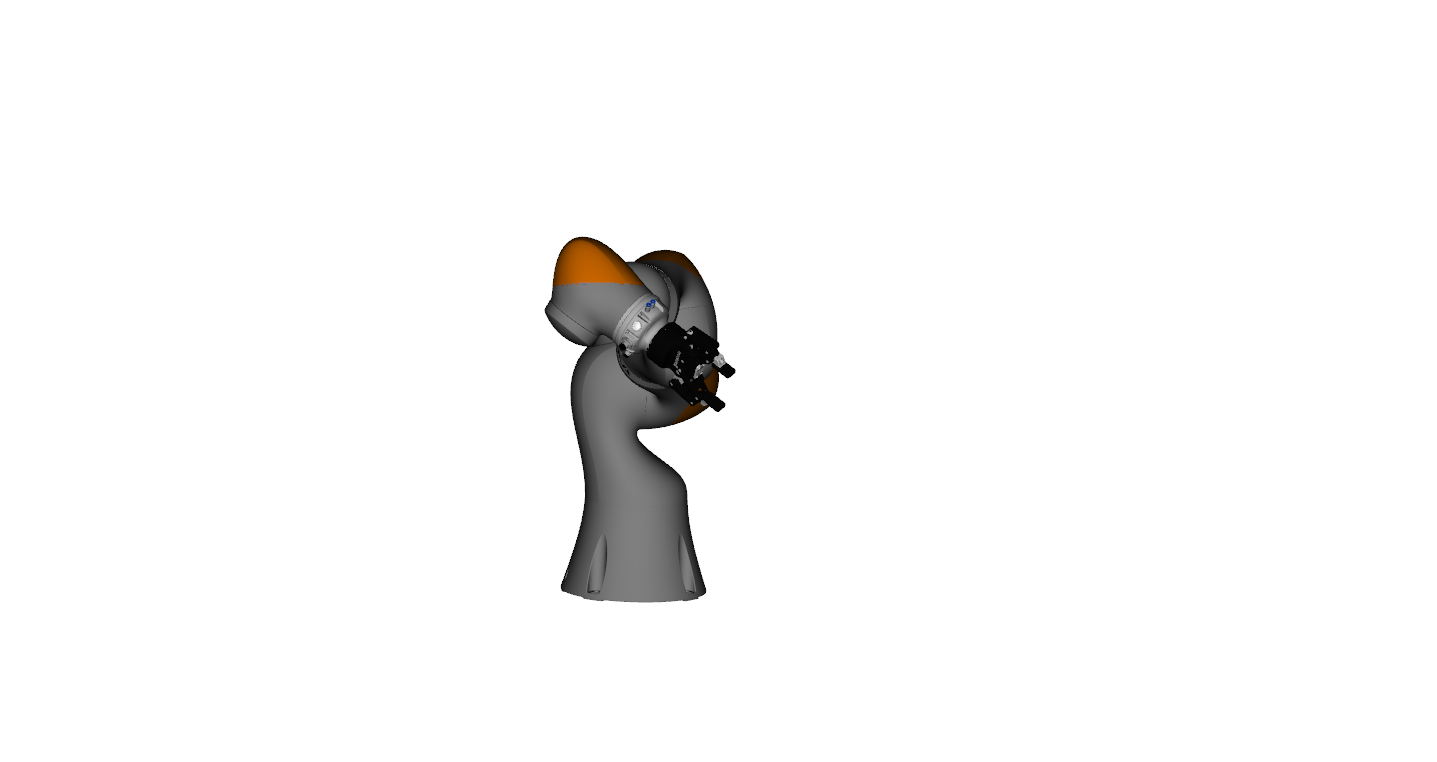}};
			\node[inner sep=0pt, opacity=0.5] (qf) at (axis cs:0, 0.135, 0.36)
			{\includegraphics[width=.415\textwidth]{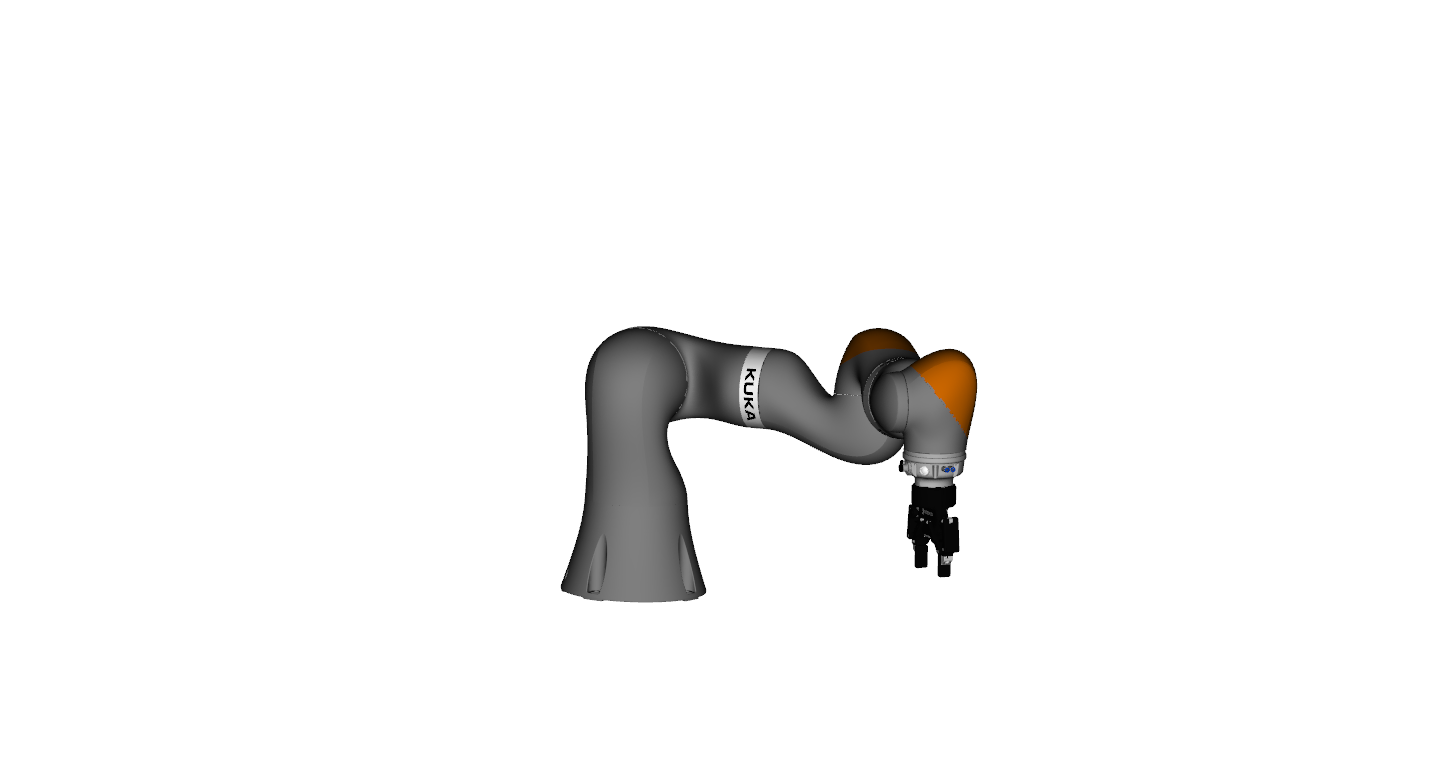}};

			\pgfplotstableread{data/p_BoundMPC.txt}\pbmpc;
			\pgfplotstableread{data/r_BoundMPC.txt}\rbmpc;
			\pgfplotstableread{data/p_SafeFlowMPC.txt}\psfmpc;
			\pgfplotstableread{data/r_SafeFlowMPC.txt}\rsfmpc;
			\pgfplotstableread{data/p_VPSTO.txt}\pgp;
			\pgfplotstableread{data/r_VPSTO.txt}\rgp;
			\addplot3 [acinred, thick]
			table [
					x expr=\thisrowno{0},
					y expr=\thisrowno{1},
					z expr=\thisrowno{2}
				] {\pgp};
			\addplot3 [acinblue, thick]
			table [
					x expr=\thisrowno{0},
					y expr=\thisrowno{1},
					z expr=\thisrowno{2}
				] {\psfmpc};
			\addplot3 [acinyellow, thick]
			table [
					x expr=\thisrowno{0},
					y expr=\thisrowno{1},
					z expr=\thisrowno{2}
				] {\pbmpc};
			%

			\addplot3 [patch,patch type=triangle, patch refines=0, shader=flat, opacity=0.15, color=acinred] coordinates
				{
					(1.1,0.05,0.7300000000000001) (0.55,-0.05,0.7300000000000001) (1.1,-0.05,0.7300000000000001)
					(1.1,0.05,0.7300000000000001) (0.55,-0.05,0.7300000000000001) (0.55,0.05,0.7300000000000001)
					(0.55,0.05,0.8699999999999999) (0.55,-0.05,0.7300000000000001) (0.55,0.05,0.7300000000000001)
					(0.55,0.05,0.8699999999999999) (0.55,-0.05,0.8699999999999999) (0.55,-0.05,0.7300000000000001)
					(0.55,0.05,0.8699999999999999) (1.1,0.05,0.7300000000000001) (0.55,0.05,0.7300000000000001)
					(0.55,0.05,0.8699999999999999) (1.1,0.05,0.7300000000000001) (1.1,0.05,0.8699999999999999)
					(1.1,-0.05,0.8699999999999999) (0.55,-0.05,0.7300000000000001) (1.1,-0.05,0.7300000000000001)
					(1.1,-0.05,0.8699999999999999) (0.55,-0.05,0.8699999999999999) (0.55,-0.05,0.7300000000000001)
					(1.1,-0.05,0.8699999999999999) (1.1,0.05,0.7300000000000001) (1.1,-0.05,0.7300000000000001)
					(1.1,-0.05,0.8699999999999999) (1.1,0.05,0.7300000000000001) (1.1,0.05,0.8699999999999999)
					(1.1,-0.05,0.8699999999999999) (0.55,0.05,0.8699999999999999) (0.55,-0.05,0.8699999999999999)
					(1.1,-0.05,0.8699999999999999) (0.55,0.05,0.8699999999999999) (1.1,0.05,0.8699999999999999)
				};
			\addplot3 [patch,patch type=triangle, patch refines=0, shader=flat, opacity=0.15, color=acinred] coordinates
				{
					(1.1,0.05,0.0) (0.55,-0.05,0.0) (1.1,-0.05,0.0)
					(1.1,0.05,0.0) (0.55,-0.05,0.0) (0.55,0.05,0.0)
					(0.55,0.05,0.35) (0.55,-0.05,0.0) (0.55,0.05,0.0)
					(0.55,0.05,0.35) (0.55,-0.05,0.35) (0.55,-0.05,0.0)
					(0.55,0.05,0.35) (1.1,0.05,0.0) (0.55,0.05,0.0)
					(0.55,0.05,0.35) (1.1,0.05,0.0) (1.1,0.05,0.35)
					(1.1,-0.05,0.35) (0.55,-0.05,0.0) (1.1,-0.05,0.0)
					(1.1,-0.05,0.35) (0.55,-0.05,0.35) (0.55,-0.05,0.0)
					(1.1,-0.05,0.35) (1.1,0.05,0.0) (1.1,-0.05,0.0)
					(1.1,-0.05,0.35) (1.1,0.05,0.0) (1.1,0.05,0.35)
					(1.1,-0.05,0.35) (0.55,0.05,0.35) (0.55,-0.05,0.35)
					(1.1,-0.05,0.35) (0.55,0.05,0.35) (1.1,0.05,0.35)
				};
			\addplot3 [patch,patch type=triangle, patch refines=0, shader=flat, opacity=0.15, color=acingreen] coordinates
				{
					(0.6643815288212902,0.2947798213329385,0.19998135338542958) (0.615362408348182,0.4886830882902628,0.09998810915703432) (0.6643857147508129,0.2947843637361393,0.09998135357620673)
					(0.6643815288212902,0.2947798213329385,0.19998135338542958) (0.615362408348182,0.4886830882902628,0.09998810915703432) (0.6153582224186596,0.48867854588706217,0.19998810896625718)
					(0.6643815288212902,0.2947798213329385,0.19998135338542958) (0.6159068476985213,0.2825239948263581,0.19997876756038113) (0.5668835412958906,0.4764227193804817,0.19998552314120874)
					(0.6643815288212902,0.2947798213329385,0.19998135338542958) (0.6153582224186596,0.48867854588706217,0.19998810896625718) (0.5668835412958906,0.4764227193804817,0.19998552314120874)
					(0.615911033628044,0.28252853722955884,0.09997876775115827) (0.6643815288212902,0.2947798213329385,0.19998135338542958) (0.6159068476985213,0.2825239948263581,0.19997876756038113)
					(0.615911033628044,0.28252853722955884,0.09997876775115827) (0.6643815288212902,0.2947798213329385,0.19998135338542958) (0.6643857147508129,0.2947843637361393,0.09998135357620673)
					(0.5668877272254134,0.4764272617836825,0.09998552333198586) (0.6153582224186596,0.48867854588706217,0.19998810896625718) (0.5668835412958906,0.4764227193804817,0.19998552314120874)
					(0.5668877272254134,0.4764272617836825,0.09998552333198586) (0.615362408348182,0.4886830882902628,0.09998810915703432) (0.6153582224186596,0.48867854588706217,0.19998810896625718)
					(0.5668877272254134,0.4764272617836825,0.09998552333198586) (0.6159068476985213,0.2825239948263581,0.19997876756038113) (0.5668835412958906,0.4764227193804817,0.19998552314120874)
					(0.5668877272254134,0.4764272617836825,0.09998552333198586) (0.615911033628044,0.28252853722955884,0.09997876775115827) (0.6159068476985213,0.2825239948263581,0.19997876756038113)
					(0.5668877272254134,0.4764272617836825,0.09998552333198586) (0.615362408348182,0.4886830882902628,0.09998810915703432) (0.6643857147508129,0.2947843637361393,0.09998135357620673)
					(0.5668877272254134,0.4764272617836825,0.09998552333198586) (0.615911033628044,0.28252853722955884,0.09997876775115827) (0.6643857147508129,0.2947843637361393,0.09998135357620673)
				};
			\addplot3[
				only marks,
				mark=*,
				mark size=1pt,
			] coordinates {(0, 0, 0)};
			\addplot3[
				only marks,
				mark=*,
				mark size=1pt,
			] coordinates {(0.61563463, 0.38560354, 0.14998344)};
			\draw (axis cs:0.61563463, 0.38560354, 0.14998344) node[
				scale=1.0,
				anchor=base west,
				text=black,
				rotate=0.0
			]{\bfseries $\vec{p}_\mathrm{f}$};
			\addplot3[
				only marks,
				mark=*,
				mark size=1pt,
			] coordinates {(0.45932693, -0.59999839 , 0.4454559 )};
			\draw (axis cs:0.45932693, -0.59999839 , 0.4054559 ) node[
				scale=1.0,
				anchor=base west,
				text=black,
				rotate=0.0
			]{\bfseries $\vec{p}_{0}$};

		\end{axis}
	\end{tikzpicture}
	\caption{Experiment 1: Environment with red obstacles and green object to grasp. An example trajectory for three planners is shown. The robot is visualized for the start and end configuration of the SafeFlowMPC trajectory.}
	\label{fig:scen_global_planner}
\end{figure}

The results of all planners described above are in~\cref{tab:exp1_comp} for the 100 evaluation trajectories in terms of average trajectory duration $T_{\mathrm{traj}}$, planning time $T_{\mathrm{plan}}$, success rate $r_{\mathrm{success}}$ for reaching the final pose without collisions, and the maximum obstacle collision $c_{\mathrm{obs}}$. The latter is the maximum obstacle penetration depth in case of a collision and zero otherwise. SafeFlowMPC shows low trajectory times $T_{\mathrm{traj}}$ comparable to the global planner VP-STO it learned from. Additionally, the robot moves through the environment without violating the constraints. The success rate $r_{\mathrm{success}}$ is higher than for the MPC-based planner BoundMPC.
BoundMPC is also safe but has considerably longer trajectory times $T_{\mathrm{traj}}$ since it only plans locally.
Behavior cloning (BC) has fast planning time $T_{\mathrm{plan}}$ since it is only executed once every time step. However, this model is unable to solve the task successfully most of the time.
The baseline flow-matching method (FM) is fast in comparison to SafeFlowMPC in terms of $T_{\mathrm{traj}}$, but has much a lower success rate $r_{\mathrm{success}}$ due to collisions.
The average planning time $T_{\mathrm{plan}}$ is fast for BoundMPC, but the time to solve its nonlinear optimization problem varies significantly. Therefore, its planning horizon needs to be chosen such that the average time is low to account for potential longer solves. SafeFlowMPC is much better at exploiting the maximum planning time of \SI{100}{\milli\second} because the iterations~\cref{eq:flow_step} are designed such that staying on the safety manifold is incentivized and a safe trajectory always exists according to~\cref{th:anytime_feasibility}.
Solving the projection problem~\cref{eq:trajectory_projection} to optimality (Ours w/ NL-Opt) leads to similar success rates as our method but incurs much higher computational cost as indicated by the average planning time $T_{\mathrm{plan}}$. The large standard deviation of $T_{\mathrm{plan}}$ is especially problematic as it frequently violates the real-time property. This is more pronounced than for BoundMPC, which also solves a nonlinear optimization problem, because the flow network presents an unknown input to the solver, which complicates the problem.
The finetuning for SafeFlowMPC, described in~\cref{ssec:training_procedure}, is very important as the model without the finetuning (Ours w/o finetuning) has a considerably lower task success $r_{\mathrm{success}}$. Note that even without the finetuning, the model remains safe, i.e., $c_{\mathrm{obs}} = 0$.
\begin{figure}
	\centering
	\addtolength\abovecaptionskip{-15pt}
	\addtolength\belowcaptionskip{-15pt}
	\def\axisdefaultwidth{\linewidth}
	\def\axisdefaultheight{0.4\linewidth}
	\begin{tikzpicture}

		\definecolor{darkgrey176}{RGB}{176,176,176}
		\definecolor{firebrick1861843}{RGB}{186,18,43}
		\definecolor{lightgrey204}{RGB}{204,204,204}
		\definecolor{sandybrown25220370}{RGB}{252,203,70}
		\definecolor{teal0101153}{RGB}{0,101,153}

		\begin{axis}[
				legend cell align={left},
				legend style={fill opacity=0.8, draw opacity=1, text opacity=1, draw=lightgrey204},
				tick align=outside,
				tick pos=left,
				x grid style={darkgrey176},
				xlabel={\(\displaystyle t\) / \si{\second}},
				xmin=0, xmax=8,
				xtick style={color=black},
				xtick={0,1,2,3,4,5,6,7,8},
				xticklabels={
						\(\displaystyle {0}\),
						\(\displaystyle {1}\),
						\(\displaystyle {2}\),
						\(\displaystyle {3}\),
						\(\displaystyle {4}\),
						\(\displaystyle {5}\),
						\(\displaystyle {6}\),
						\(\displaystyle {7}\),
						\(\displaystyle {8}\)
					},
				y grid style={darkgrey176},
				ylabel={\(\displaystyle \max{\lvert \dot{\mathbf{q}} / \overline{\dot{\mathbf{q}}}\rvert}\)},
				ymin=0, ymax=1,
				ytick style={color=black},
				ytick={0,0.2,0.4,0.6,0.8,1},
				yticklabels={
						\(\displaystyle {0.0}\),
						\(\displaystyle {0.2}\),
						\(\displaystyle {0.4}\),
						\(\displaystyle {0.6}\),
						\(\displaystyle {0.8}\),
						\(\displaystyle {1.0}\)
					}
			]
			\addplot [semithick, firebrick1861843]
			table {%
					0 1.59960176077051e-16
					0.1 0.163575596286056
					0.2 0.287535830717192
					0.3 0.371880703293405
					0.4 0.416738965090424
					0.5 0.448923073091954
					0.6 0.488661435460804
					0.7 0.53595405219697
					0.8 0.590766887890672
					0.9 0.651148698785838
					1 0.715967436273408
					1.1 0.78522310035338
					1.2 0.858257866424956
					1.3 0.917062823488962
					1.4 0.953526649144367
					1.5 0.96764934339117
					1.6 0.959661511353804
					1.7 0.933302661254402
					1.8 0.889883002239331
					1.9 0.829402534308592
					2 0.754867247102971
					2.1 0.698564593805128
					2.2 0.693568092092444
					2.3 0.729222900160487
					2.4 0.728064683822542
					2.5 0.692645303694722
					2.6 0.62349424340295
					2.7 0.520611502947227
					2.8 0.391465017828783
					2.9 0.277915431952849
					3 0.187968628762231
					3.1 0.134072617645871
					3.2 0.10121211916776
					3.3 0.102082067231503
					3.4 0.100509535299573
					3.5 0.0964945233719736
					3.6 0.101358659943547
					3.7 0.102308648440689
					3.8 0.0961213975427381
					3.9 0.0827969072496935
					4 0.0623351775615556
					4.1 0.0347362084783244
					4.2 5.98879632091247e-16
					4.3 0
					4.4 0
					4.5 0
					4.6 0
					4.7 0
				};
			\addlegendentry{VPSTO}
			\addplot [semithick, teal0101153]
			table {%
					0 0
					0.1 0.0184306058120291
					0.2 0.117721200741964
					0.3 0.261552196780623
					0.4 0.397618403526713
					0.5 0.506486577829738
					0.6 0.590764073208293
					0.7 0.65825608405293
					0.8 0.71394611531084
					0.9 0.759797876878836
					1 0.792970741234942
					1.1 0.80996517178587
					1.2 0.810595662665314
					1.3 0.804303862320303
					1.4 0.891712554678781
					1.5 0.955163335965501
					1.6 0.986871905203069
					1.7 0.994560552181152
					1.8 0.994409919690947
					1.9 0.994230807424378
					2 0.991697602405147
					2.1 0.980618521771079
					2.2 0.949896243061797
					2.3 0.884601421436434
					2.4 0.786229258963299
					2.5 0.666032153711942
					2.6 0.536980114803228
					2.7 0.41044244418602
					2.8 0.294928002570584
					2.9 0.209225381873821
					3 0.189148075489337
					3.1 0.16489704022925
					3.2 0.140127298164869
					3.3 0.117707934225807
					3.4 0.0991014060318503
					3.5 0.0843183282939812
					3.6 0.0723970500153582
					3.7 0.0621714858853006
					3.8 0.0520009047524102
					3.9 0.0406899776496183
					4 0.0287140174303624
					4.1 0.0171919542517975
					4.2 0.0095742657552814
					4.3 0.00727007277224308
					4.4 0.00737814812182608
					4.5 0.00929031285273472
					4.6 0.0103916684712403
					4.7 0.0100069297136437
					4.8 0.00866115498370861
					4.9 0.00715532687120661
					5 0.00643019808381354
					5.1 0.0057018985523528
					5.2 0.0064345750609916
					5.3 0.00719059621879637
					5.4 0.00754861975055437
					5.5 0.00760009340997059
					5.6 0.00741121789007148
					5.7 0.00702826582771128
					5.8 0.00658921419313162
					5.9 0.00616283377779882
					6 0.00580125675782014
					6.1 0.0055146113027686
					6.2 0.00528639665929349
					6.3 0.00513117737058962
					6.4 0.00506946352720995
					6.5 0.00508093230036502
					6.6 0.00510652797239045
					6.7 0.00514486997797434
					6.8 0.00520307052545182
					6.9 0.00526730378671983
					7 0.00531788035645835
					7.1 0.00533850827472467
					7.2 0.00533385310431345
					7.3 0.00532171839991218
					7.4 0.00531520307945623
					7.5 0.00531559714461196
					7.6 0.00526352963918372
					7.7 0.00519053900363628
					7.8 0.00513547621375277
				};
			\addlegendentry{SafeFlowMPC}
			\addplot [semithick, sandybrown25220370]
			table {%
					0 0
					0.1 0.0189170699172659
					0.2 0.112830761935179
					0.3 0.216218908826957
					0.4 0.286179616594341
					0.5 0.394533269205733
					0.6 0.4843959586786
					0.7 0.548808599054389
					0.8 0.585331468708621
					0.9 0.595594006298782
					1 0.584692539107179
					1.1 0.556061403316312
					1.2 0.514267070415718
					1.3 0.464287399612016
					1.4 0.493491273577633
					1.5 0.536526415598671
					1.6 0.570763788393849
					1.7 0.59591294443325
					1.8 0.610876908632745
					1.9 0.614510546844379
					2 0.638635050980325
					2.1 0.670322439836716
					2.2 0.667005890318555
					2.3 0.632287705933785
					2.4 0.579577160345959
					2.5 0.518654879674867
					2.6 0.453015772011219
					2.7 0.385100406508632
					2.8 0.318367027353017
					2.9 0.259864086323607
					3 0.238821922421826
					3.1 0.226638558048889
					3.2 0.238676052260487
					3.3 0.247443304755871
					3.4 0.258081981856277
					3.5 0.299546593363477
					3.6 0.395690693536001
					3.7 0.497865090976086
					3.8 0.595314676755731
					3.9 0.67871143012057
					4 0.741258174054494
					4.1 0.779117259054743
					4.2 0.791098527788559
					4.3 0.778205574407111
					4.4 0.743069670923736
					4.5 0.68956529243993
					4.6 0.622267016810133
					4.7 0.545930245861808
					4.8 0.465099413159568
					4.9 0.383815826433503
					5 0.305480498527305
					5.1 0.232790567609864
					5.2 0.167713389447921
					5.3 0.111493027562977
					5.4 0.064692639860822
					5.5 0.0523314295954167
					5.6 0.051935074399732
					5.7 0.0512776455182471
					5.8 0.0504572686283
					5.9 0.0495269320652794
					6 0.0484947843723173
					6.1 0.0473341054214997
					6.2 0.0460039403070581
					6.3 0.044476469286563
					6.4 0.0427258196174686
					6.5 0.0407422669894706
					6.6 0.0385389610404809
					6.7 0.0361503153171474
					6.8 0.0336272797661989
					6.9 0.0310306627512098
					7 0.028423787932638
					7.1 0.0258663219606048
					7.2 0.0234094458458194
					7.3 0.021093208615977
					7.4 0.0189454985942197
					7.5 0.0169827726156464
					7.6 0.0152113189188659
					7.7 0.0136294638610811
					7.8 0.0122298941326102
					7.9 0.0110015985249087
					8 0.00993144785169445
					8.1 0.00900537032872701
					8.2 0.00820911502063671
					8.3 0.00752878441543252
					8.4 0.0069509720995077
					8.5 0.00646283983394427
					8.6 0.00605225672735002
					8.7 0.00570788492602645
					8.8 0.0054192236583062
					8.9 0.00517673625487433
					9 0.00497183180613615
					9.1 0.00479706166148649
					9.2 0.00464596938922985
					9.3 0.00451308013290386
					9.4 0.00439394685608783
					9.5 0.00428494895515479
					9.6 0.00418324885025901
					9.7 0.0040866904517219
					9.8 0.00399364510751681
					9.9 0.00390298706588652
					10 0.00381396846275148
					10.1 0.00372607439103981
					10.2 0.00363902495371585
				};
			\addlegendentry{BoundMPC}
		\end{axis}

	\end{tikzpicture}
	\caption{Experiment 1: Maximum values of the normalized joint velocity for the example trajectories in~\cref{fig:scen_global_planner}.}
	\label{fig:exp1_joint_vel}
	\vspace{-15pt}
\end{figure}
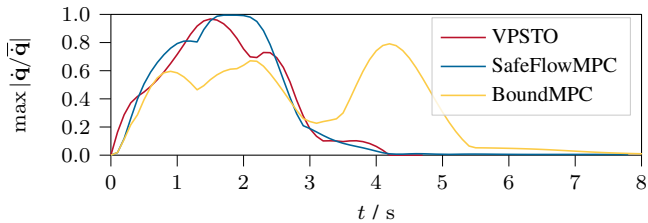
Furthermore, the maximum of the normalized joint velocities in~\cref{fig:exp1_joint_vel} shows that SafeFlowMPC is able to exploit the full range of the joint velocities without violating the limits. This leads to a shorter trajectory compared to BoundMPC.
\begin{figure}
	\centering
	\addtolength\abovecaptionskip{-15pt}
	\def\axisdefaultwidth{\linewidth}
	\def\axisdefaultheight{0.4\linewidth}
	\begin{tikzpicture}

		\definecolor{darkgrey176}{RGB}{176,176,176}
		\definecolor{firebrick1861843}{RGB}{186,18,43}
		\definecolor{lightgrey204}{RGB}{204,204,204}
		\definecolor{sandybrown25220370}{RGB}{252,203,70}
		\definecolor{teal0101153}{RGB}{0,101,153}

		\begin{axis}[
				legend cell align={left},
				legend style={fill opacity=0.8, draw opacity=1, text opacity=1, draw=lightgrey204},
				tick align=outside,
				tick pos=left,
				x grid style={darkgrey176},
				xlabel={\(\displaystyle t\) / \si{\second}},
				xmin=0, xmax=7.7,
				xtick style={color=black},
				xtick={0,1,2,3,4,5,6,7,8},
				xticklabels={
						\(\displaystyle {0}\),
						\(\displaystyle {1}\),
						\(\displaystyle {2}\),
						\(\displaystyle {3}\),
						\(\displaystyle {4}\),
						\(\displaystyle {5}\),
						\(\displaystyle {6}\),
						\(\displaystyle {7}\),
						\(\displaystyle {8}\)
					},
				y grid style={darkgrey176},
				ylabel={Norm of terminal states},
				ymin=-0.0148186870947038, ymax=0.312674287021838,
				ytick style={color=black},
				ytick={-0.05,0,0.05,0.1,0.15,0.2,0.25,0.3,0.35},
				yticklabels={
						\(\displaystyle {\ensuremath{-}0.05}\),
						\(\displaystyle {0.00}\),
						\(\displaystyle {0.05}\),
						\(\displaystyle {0.10}\),
						\(\displaystyle {0.15}\),
						\(\displaystyle {0.20}\),
						\(\displaystyle {0.25}\),
						\(\displaystyle {0.30}\),
						\(\displaystyle {0.35}\)
					}
			]
			\addplot [semithick, firebrick1861843]
			table {%
					0 0.161268893896281
					0.1 0.15219352149783
					0.2 0.1502646081419
					0.3 0.153917813600201
					0.4 0.161458524088986
					0.5 0.175401458917598
					0.6 0.188121794263868
					0.7 0.196857435734934
					0.8 0.194779739833628
					0.9 0.198634693210228
					1 0.201633466617285
					1.1 0.297788242743814
					1.2 0.240183365611914
					1.3 0.180629913890048
					1.4 0.105737418560784
					1.5 0.0565221038469561
					1.6 0.0281324408127514
					1.7 0.0227566000560884
					1.8 0.0237636752503771
					1.9 0.0232500233549874
					2 0.0247175235516338
					2.1 0.0250757652363189
					2.2 0.0225997837181112
					2.3 0.0178383228444813
					2.4 0.0155411112039225
					2.5 0.0151148062932153
					2.6 0.01495069431563
					2.7 0.0144255487931616
					2.8 0.0127706536922388
					2.9 0.0112742139993536
					3 0.0105681944804914
					3.1 0.0107098092067237
					3.2 0.0108322800654462
					3.3 0.0101986478960379
					3.4 0.00866215310052327
					3.5 0.00697631646149814
					3.6 0.00656062539100674
					3.7 0.00745111966994885
					3.8 0.00848525625272594
					3.9 0.0090050683485861
					4 0.00915480203274148
					4.1 0.00931598309204533
					4.2 0.00913445410180616
					4.3 0.0089866053143976
					4.4 0.00874160878053256
					4.5 0.00857760253864531
					4.6 0.00834088889807995
					4.7 0.00815861564178112
					4.8 0.0078732912278285
					4.9 0.0079010035246116
					5 0.00782469733299145
					5.1 0.00784039830375427
					5.2 0.00774203094613731
					5.3 0.00792361077434738
					5.4 0.00796890932543626
					5.5 0.00814511851156812
					5.6 0.00820761436993212
					5.7 0.00813958499478818
					5.8 0.00808250929923531
					5.9 0.0081293008068002
					6 0.00804611823923148
					6.1 0.00819354121022326
					6.2 0.00812598921517652
					6.3 0.00797017319372911
					6.4 0.00798833181676813
					6.5 0.00791608808787933
					6.6 0.00791957069947483
					6.7 0.00800546375382453
					6.8 0.00791080218986105
					6.9 0.00799446647020529
					7 0.00794972290606129
					7.1 0.00784036161875959
					7.2 0.00789861628581216
					7.3 0.00793541783397115
					7.4 0.00792942825085929
					7.5 0.00789566743193199
					7.6 0.00789975970853637
					7.7 0.00787136502530157
				};
			\addlegendentry{\(\displaystyle \lVert \dot{\mathbf{q}}_{T} \rVert_2\) / \si{\degree\per\second}}
			\addplot [semithick, teal0101153]
			table {%
					0 0.046590377093494
					0.1 0.042956342330676
					0.2 0.0415834382148833
					0.3 0.0421804316911703
					0.4 0.0441563409590078
					0.5 0.048627046799293
					0.6 0.0526748145671598
					0.7 0.0561724955274299
					0.8 0.0550880042102347
					0.9 0.0548107393859186
					1 0.0564005300610931
					1.1 0.0890421885658178
					1.2 0.0747098982043929
					1.3 0.0566642876078771
					1.4 0.0357073047496852
					1.5 0.0225083945357945
					1.6 0.013016715630521
					1.7 0.00760884083717528
					1.8 0.00599325588103632
					1.9 0.0050419291277444
					2 0.00499865863015934
					2.1 0.00504375671034333
					2.2 0.0047015925251516
					2.3 0.00405348299426699
					2.4 0.00403410821711801
					2.5 0.00425044266787661
					2.6 0.00431528395241389
					2.7 0.00414595496837507
					2.8 0.00365425992646386
					2.9 0.00316079627449927
					3 0.00281755941727658
					3.1 0.00268770935032175
					3.2 0.00260296309997812
					3.3 0.00239235304493522
					3.4 0.00201308504779259
					3.5 0.00162176729167963
					3.6 0.00156292951280454
					3.7 0.00186360297893772
					3.8 0.00216245900413911
					3.9 0.00230628648701908
					4 0.00235284421751925
					4.1 0.00237196246802432
					4.2 0.00231811953136825
					4.3 0.00225843553325729
					4.4 0.00217948421748467
					4.5 0.00212180751942064
					4.6 0.00205547351816115
					4.7 0.00200152931389785
					4.8 0.00193783982575481
					4.9 0.0019326232033241
					5 0.00192441234733112
					5.1 0.00193545733916697
					5.2 0.00191845047702597
					5.3 0.00196468482436795
					5.4 0.00197761850780289
					5.5 0.00201861906216504
					5.6 0.00203389280754928
					5.7 0.00202601915290103
					5.8 0.00201670908571038
					5.9 0.00202548745531718
					6 0.00200412389105146
					6.1 0.00203233303663027
					6.2 0.00201842266921201
					6.3 0.00198508052683883
					6.4 0.00198710088721219
					6.5 0.00197437888149765
					6.6 0.00197255383289035
					6.7 0.00198944779358043
					6.8 0.00197017304121165
					6.9 0.00198700504971151
					7 0.00197514116195914
					7.1 0.00195800946144506
					7.2 0.00196616947462781
					7.3 0.00197340126982148
					7.4 0.00197236072213331
					7.5 0.00196595052612592
					7.6 0.00197082954042898
					7.7 0.00196300949238003
				};
			\addlegendentry{\(\displaystyle \lVert \ddot{\mathbf{q}}_{T} \rVert_2\) / \si{\degree\per\second\squared}}
			\addplot [semithick, sandybrown25220370]
			table {%
					0 0.00206130250033589
					0.1 0.00189455447808024
					0.2 0.00182905830225893
					0.3 0.00185283652563661
					0.4 0.00193916739856299
					0.5 0.00213957638141135
					0.6 0.00232091128843915
					0.7 0.00248187952507476
					0.8 0.00243276725973491
					0.9 0.0024121660165799
					1 0.00249064443577979
					1.1 0.00395772377292505
					1.2 0.00333797479414365
					1.3 0.00253665043112048
					1.4 0.00161277210520188
					1.5 0.00103631830014641
					1.6 0.00061391487592283
					1.7 0.000361326443601376
					1.8 0.00027495131274602
					1.9 0.000222459851633926
					2 0.000212361848793649
					2.1 0.000211643210407065
					2.2 0.000198032386237429
					2.3 0.000173940792290672
					2.4 0.000176891635200715
					2.5 0.000188184537689583
					2.6 0.000191564388291151
					2.7 0.000183901467876672
					2.8 0.000162064312468814
					2.9 0.000139856947953618
					3 0.000123828675844091
					3.1 0.000116962310184501
					3.2 0.000112376540203956
					3.3 0.000102800996093023
					3.4 8.63547328479898e-05
					3.5 6.96026880464389e-05
					3.6 6.73571833208558e-05
					3.7 8.08757432799853e-05
					3.8 9.40648965732046e-05
					3.9 0.000100376429271204
					4 0.000102439244950905
					4.1 0.000103123880714262
					4.2 0.000100728621681298
					4.3 9.79927880237708e-05
					4.4 9.44588575416276e-05
					4.5 9.18526989597799e-05
					4.6 8.89359030501189e-05
					4.7 8.65500865473362e-05
					4.8 8.38414757650476e-05
					4.9 8.35461139048901e-05
					5 8.32609459103451e-05
					5.1 8.37849397582239e-05
					5.2 8.3100446875921e-05
					5.3 8.51062049867391e-05
					5.4 8.56784868587347e-05
					5.5 8.74343422082886e-05
					5.6 8.80912904316164e-05
					5.7 8.78042152135924e-05
					5.8 8.7431433713993e-05
					5.9 8.77925709692208e-05
					6 8.68651790333326e-05
					6.1 8.80330801851832e-05
					6.2 8.74473697865322e-05
					6.3 8.60383890178355e-05
					6.4 8.61103600895052e-05
					6.5 8.55934056593324e-05
					6.6 8.54973056728749e-05
					6.7 8.62002634397096e-05
					6.8 8.53932679071026e-05
					6.9 8.6096191912087e-05
					7 8.55792134660161e-05
					7.1 8.49002908405727e-05
					7.2 8.52148474165366e-05
					7.3 8.55148252003865e-05
					7.4 8.54726443368724e-05
					7.5 8.52090112829156e-05
					7.6 8.54419590004554e-05
					7.7 8.50994049412676e-05
				};
			\addlegendentry{\(\displaystyle \lVert \dddot{\mathbf{q}}_{T} \rVert_2\) / \si{\degree\per\second\cubed}}
		\end{axis}

	\end{tikzpicture}
	\caption{Experiment 1: Norm of joint velocity $\dot{\vec{q}}_{T}$, acceleration $\ddot{\vec{q}}_{T}$, and jerk $\dddot{\vec{q}}_{T}$ at the end of each planning horizon using SafeFlowMPC for the example trajectory in~\cref{fig:scen_global_planner}.}
	\label{fig:exp1_joint_term}
\end{figure}
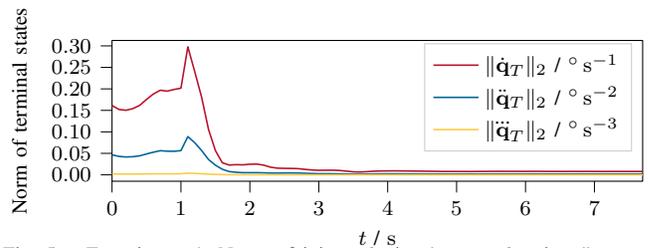
Additionally, the joint states at the end of each planning horizon, depicted in~\cref{fig:exp1_joint_term}, are close to zero such that the horizon always ends in the safe terminal set $\mathcal{S}_{T}$ according to~\cref{eq:robot_g}. The states are not exactly zero as this is difficult to enforce but sufficiently low for the low-level joint controller to stabilize around the final point as demonstrated in the supplementary video.

\begin{table}
	\caption{Experiment 1: Results of different planners in terms of average trajectory duration $T_{\mathrm{traj}}$, average planning time $T_{\mathrm{plan}}$, obstacle collisions $c_{\mathrm{obs}}>0$, and success rate $r_{\mathrm{success}}$}\label{tab:exp1_comp}
	\centering
	\begin{tabular}[c]{p{1.5cm}cccc}
		Method              & $T_{\mathrm{traj}}$ / \si{\second} & $T_{\mathrm{plan}}$ / \si{\milli\second} & $c_{\mathrm{obs}}$ / \si{\meter} & $r_{\mathrm{success}}$ \\
		\hline
		VP-STO              & $4.46$                             & $11400$ (offline)                        & $0$                              & \SI{100}{\percent}     \\
		BoundMPC            & $6.16$                             & $35 \pm 10$                              & $0$                              & \SI{76}{\percent}      \\
		BC                  & $7.6$                              & $1 \pm 1$                                & $0.08$                           & \SI{2}{\percent}       \\
		FM                  & $5.77$                             & $33 \pm 2$                               & $0.08$                           & \SI{22}{\percent}      \\
		Ours w/ NL-Opt      & $5.69$                             & $93 \pm 178$                             & $0$                              & $\SI{87}{\percent}$    \\
		Ours w/o finetuning & $6.40$                             & $64 \pm 5$                               & $0$                              & \SI{58}{\percent}      \\
		SafeFlowMPC (Ours)  & $5.12$                             & $62 \pm 5$                               & $0$                              & \SI{86}{\percent}      \\
		\hline
	\end{tabular}
\end{table}

\subsection{Experiment 2: Online replanning for object grasps}
\label{sec:global_local_replanning}

The reactivity of SafeFlowMPC is compared in the environment shown in ~\cref{fig:scen_global_planner} by changing the pose of the object during the motion three times randomly. The results for the real-time capable methods are reported in~\cref{tab:exp1_comp_replan}. A motion is counted as successful if it reaches the last replanned object pose and has no collisions. Generally, the drop in success rate for all methods indicates that this task is more challenging than Experiment 1 in~\cref{sec:global_local_planning}. SafeFlowMPC is the most successful method and is able to replan motions in real time without any collisions in \SI{82}{\percent} of the trails.

Thus, the SafeFlowMPC formulation enables fast and safe trajectory planning with similar quality as global trajectory planners. At the same time, it allows for more reactivity due to the online replanning.

\begin{table}
	\caption{Experiment 2: Results of different planners in terms of average planning time $T_{\mathrm{plan}}$, obstacle collisions $c_{\mathrm{obs}}>0$, and success rate $r_{\mathrm{success}}$}\label{tab:exp1_comp_replan}
	\centering
	\begin{tabular}[c]{lccc}
		Method              & $T_{\mathrm{plan}}$ / \si{\milli\second} & $c_{\mathrm{obs}}$ / \si{\meter} & $r_{\mathrm{success}}$ \\
		\hline
		BoundMPC            & $36 \pm 11$                              & $0$                              & \SI{74}{\percent}      \\
		BC                  & $1 \pm 1$                                & $0.08$                           & \SI{0}{\percent}       \\
		FM                  & $33 \pm 2$                               & $0.08$                           & \SI{15}{\percent}      \\
		Ours w/o finetuning & $62 \pm 5$                               & $0$                              & \SI{52}{\percent}      \\
		SafeFlowMPC (Ours)  & $63 \pm 5$                               & $0$                              & \SI{82}{\percent}      \\
		\hline
	\end{tabular}
\end{table}

\subsection{Experiment 3: Dynamic human-robot object handover}
\label{sec:handover}

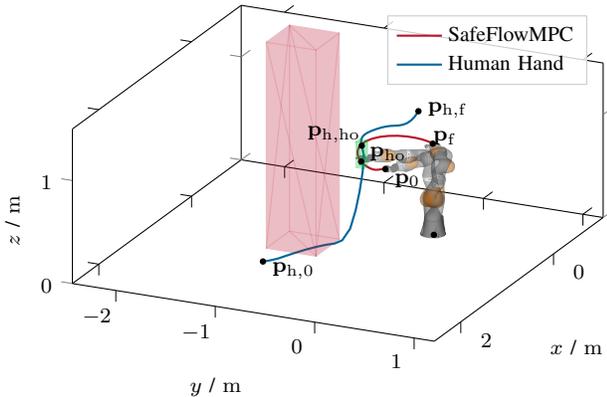
\begin{figure}
	\centering
	\addtolength\abovecaptionskip{-15pt}
	\def\axisdefaultwidth{\linewidth}
	\def\axisdefaultheight{0.7\linewidth}
	\begin{tikzpicture}
		\definecolor{darkgray176}{RGB}{176,176,176}
		\definecolor{acinblue}{RGB}{0,102,153}
		\definecolor{acinyellow}{RGB}{252, 204, 71}
		\definecolor{acingreen}{RGB}{0,190,65}
		\definecolor{acinred}{RGB}{186,18,43}
		\definecolor{lightgray204}{RGB}{204,204,204}

		\begin{axis}[
				view={115}{20},
				xlabel={$x$ / \si{\meter}},
				ylabel={$y$ / \si{\meter}},
				zlabel={$z$ / \si{\meter}},
				zmin=0,
				zmax=1.5,
				xmin=-1,
				xmax=2.5,
				axis equal,
				tick align=outside,
				tick pos=left,
				x grid style={darkgray176},
				y grid style={darkgray176},
				xtick style={color=black},
				ytick style={color=black},
				legend cell align={left},
				legend style={
						fill opacity=0.8,
						draw opacity=1,
						text opacity=1,
						draw=lightgray204,
					},
				legend pos=north east,
				legend entries ={SafeFlowMPC, Human Hand}
			]

			\node[inner sep=0pt, opacity=0.5] (qf) at (axis cs:0.05, 0.03, 0.79)
			{\includegraphics[width=.146\textwidth]{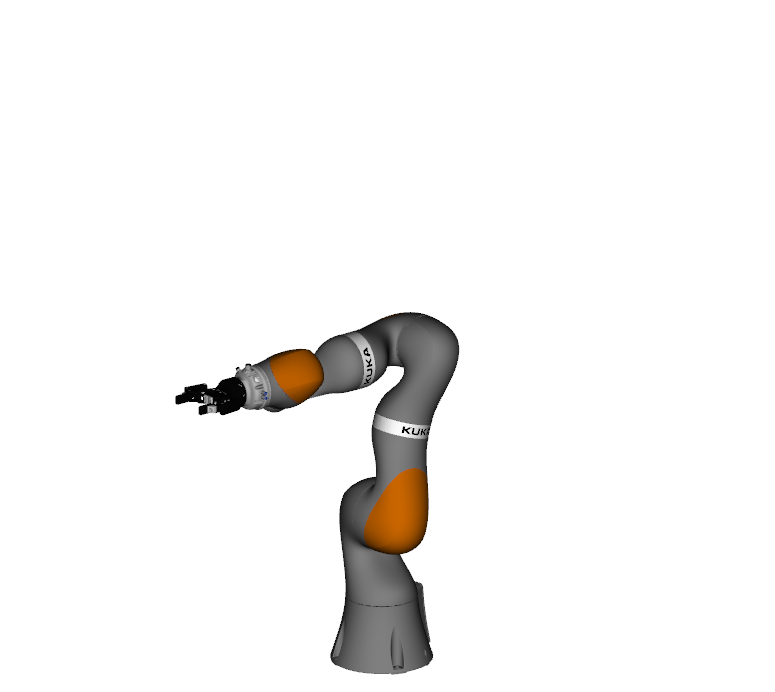}};
			\node[inner sep=0pt, opacity=0.5] (qf) at (axis cs:0.05, 0.03, 0.79)
			{\includegraphics[width=.146\textwidth]{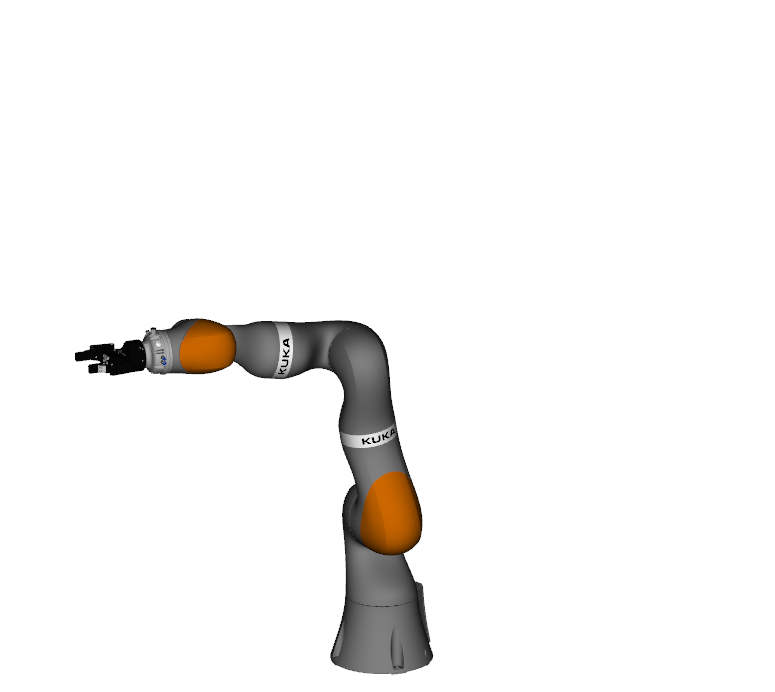}};
			\node[inner sep=0pt, opacity=0.5] (qf) at (axis cs:0.05, 0.03, 0.79)
			{\includegraphics[width=.146\textwidth]{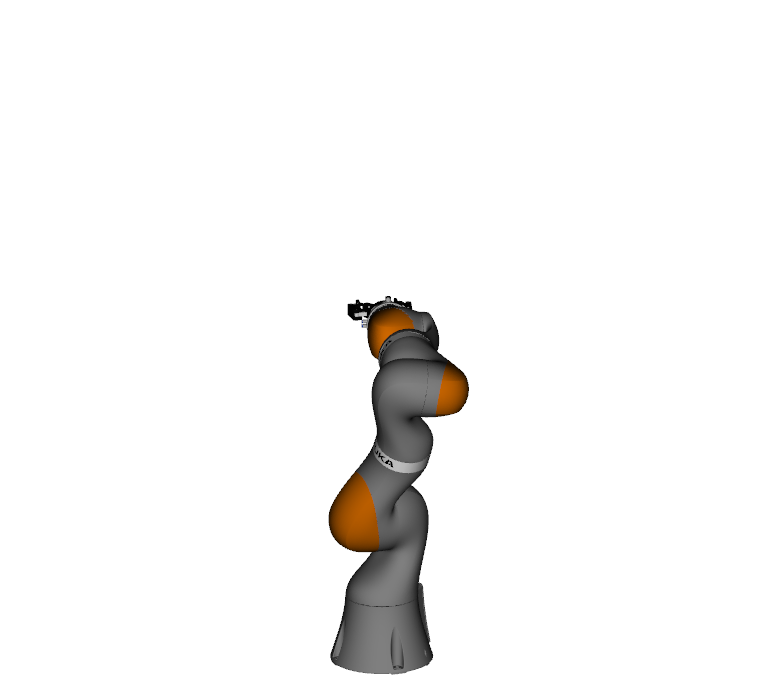}};

			\pgfplotstableread{data/p_SafeFlowMPC_ho.txt}\psfmpc;
			\pgfplotstableread{data/r_SafeFlowMPC_ho.txt}\rsfmpc;
			\pgfplotstableread{data/p_hand.txt}\phand;
			\addplot3 [acinred, thick]
			table [
					x expr=\thisrowno{0},
					y expr=\thisrowno{1},
					z expr=\thisrowno{2}
				] {\psfmpc};
			\addplot3 [acinblue, thick]
			table [
					x expr=\thisrowno{0},
					y expr=\thisrowno{1},
					z expr=\thisrowno{2}
				] {\phand};
			%

			\addplot3 [patch,patch type=triangle, patch refines=0, shader=flat, opacity=0.15, color=acinred] coordinates
				{
					(0.9686957182973392,-0.7330219075336785,0.0) (0.9686957182973394,-1.2330219075336786,2.0) (0.9686957182973392,-0.7330219075336785,2.0)
					(0.9686957182973392,-0.7330219075336785,0.0) (0.9686957182973394,-1.2330219075336786,2.0) (0.9686957182973394,-1.2330219075336786,0.0)
					(0.9686957182973392,-0.7330219075336785,0.0) (0.4686957182973392,-0.7330219075336785,-0.0) (0.4686957182973392,-1.2330219075336784,-0.0)
					(0.9686957182973392,-0.7330219075336785,0.0) (0.9686957182973394,-1.2330219075336786,0.0) (0.4686957182973392,-1.2330219075336784,-0.0)
					(0.4686957182973392,-0.7330219075336785,2.0) (0.9686957182973392,-0.7330219075336785,0.0) (0.4686957182973392,-0.7330219075336785,-0.0)
					(0.4686957182973392,-0.7330219075336785,2.0) (0.9686957182973392,-0.7330219075336785,0.0) (0.9686957182973392,-0.7330219075336785,2.0)
					(0.4686957182973392,-1.2330219075336784,2.0) (0.9686957182973394,-1.2330219075336786,0.0) (0.4686957182973392,-1.2330219075336784,-0.0)
					(0.4686957182973392,-1.2330219075336784,2.0) (0.9686957182973394,-1.2330219075336786,2.0) (0.9686957182973394,-1.2330219075336786,0.0)
					(0.4686957182973392,-1.2330219075336784,2.0) (0.4686957182973392,-0.7330219075336785,-0.0) (0.4686957182973392,-1.2330219075336784,-0.0)
					(0.4686957182973392,-1.2330219075336784,2.0) (0.4686957182973392,-0.7330219075336785,2.0) (0.4686957182973392,-0.7330219075336785,-0.0)
					(0.4686957182973392,-1.2330219075336784,2.0) (0.9686957182973394,-1.2330219075336786,2.0) (0.9686957182973392,-0.7330219075336785,2.0)
					(0.4686957182973392,-1.2330219075336784,2.0) (0.4686957182973392,-0.7330219075336785,2.0) (0.9686957182973392,-0.7330219075336785,2.0)
				};
			\addplot3 [patch,patch type=triangle, patch refines=0, shader=flat, opacity=0.15, color=acingreen] coordinates
				{
					(0.5586957182973392,-0.44302190753367837,0.7566626719988508) (0.47869571829733926,-0.5230219075336784,0.7566626719988508) (0.5586957182973393,-0.5230219075336784,0.7566626719988508)
					(0.5586957182973392,-0.44302190753367837,0.7566626719988508) (0.47869571829733926,-0.5230219075336784,0.7566626719988508) (0.47869571829733926,-0.44302190753367837,0.7566626719988508)
					(0.47869571829733926,-0.44302190753367837,1.006662671998851) (0.47869571829733926,-0.5230219075336784,0.7566626719988508) (0.47869571829733926,-0.44302190753367837,0.7566626719988508)
					(0.47869571829733926,-0.44302190753367837,1.006662671998851) (0.47869571829733926,-0.5230219075336784,1.006662671998851) (0.47869571829733926,-0.5230219075336784,0.7566626719988508)
					(0.47869571829733926,-0.44302190753367837,1.006662671998851) (0.5586957182973392,-0.44302190753367837,0.7566626719988508) (0.47869571829733926,-0.44302190753367837,0.7566626719988508)
					(0.47869571829733926,-0.44302190753367837,1.006662671998851) (0.5586957182973392,-0.44302190753367837,0.7566626719988508) (0.5586957182973392,-0.44302190753367837,1.006662671998851)
					(0.5586957182973393,-0.5230219075336784,1.006662671998851) (0.47869571829733926,-0.5230219075336784,0.7566626719988508) (0.5586957182973393,-0.5230219075336784,0.7566626719988508)
					(0.5586957182973393,-0.5230219075336784,1.006662671998851) (0.47869571829733926,-0.5230219075336784,1.006662671998851) (0.47869571829733926,-0.5230219075336784,0.7566626719988508)
					(0.5586957182973393,-0.5230219075336784,1.006662671998851) (0.5586957182973392,-0.44302190753367837,0.7566626719988508) (0.5586957182973393,-0.5230219075336784,0.7566626719988508)
					(0.5586957182973393,-0.5230219075336784,1.006662671998851) (0.5586957182973392,-0.44302190753367837,0.7566626719988508) (0.5586957182973392,-0.44302190753367837,1.006662671998851)
					(0.5586957182973393,-0.5230219075336784,1.006662671998851) (0.47869571829733926,-0.44302190753367837,1.006662671998851) (0.47869571829733926,-0.5230219075336784,1.006662671998851)
					(0.5586957182973393,-0.5230219075336784,1.006662671998851) (0.47869571829733926,-0.44302190753367837,1.006662671998851) (0.5586957182973392,-0.44302190753367837,1.006662671998851)
				};

			\addplot3[
				only marks,
				mark=*,
				mark size=1pt,
			] coordinates {(0, 0, 0)};
			\addplot3[
				only marks,
				mark=*,
				mark size=1pt,
			] coordinates{(2.32691068649292, -0.6337637901306152, 0.41110029220581057)};
			\draw (axis cs:2.32691068649292, -0.6337637901306152, 0.31110029220581057) node[
				scale=1.0,
				anchor=base west,
				text=black,
				rotate=0.0
			]{\bfseries $\vec{p}_\mathrm{h, 0}$};

			\addplot3[
				only marks,
				mark=*,
				mark size=1pt,
			] coordinates{(0.5186957182973392, -0.48302190753367835, 0.961662671998850)};
			\draw (axis cs:0.5186957182973392, -1.12302190753367835, 0.9616626719988509) node[
				scale=1.0,
				anchor=base west,
				text=black,
				rotate=0.0
			]{\bfseries $\vec{p}_\mathrm{h, ho}$};

			\addplot3[
				only marks,
				mark=*,
				mark size=1pt,
			] coordinates{(-0.9067300803868508, -0.5777403559832155, 0.8118041792813402)};
			\draw (axis cs:-0.9067300803868508, -0.5777403559832155, 0.8118041792813402) node[
				scale=1.0,
				anchor=base west,
				text=black,
				rotate=0.0
			]{\bfseries $\vec{p}_\mathrm{h, f}$};

			\addplot3[
				only marks,
				mark=*,
				mark size=1pt,
			] coordinates{(0.49353350694834774, -0.2566223602427127, 0.7606581156393782)};
			\draw (axis cs:0.49353350694834774, -0.2566223602427127, 0.6206581156393782) node[
				scale=1.0,
				anchor=base west,
				text=black,
				rotate=0.0
			]{\bfseries $\vec{p}_\mathrm{0}$};

			\addplot3[
				only marks,
				mark=*,
				mark size=1pt,
			] coordinates{(0.5761461703654832, -0.4627136540583015, 0.8297219928918769)};
			\draw (axis cs:0.4890795894423669, -0.5081318827242907, 0.8437911848078921) node[
				scale=1.0,
				anchor=base west,
				text=black,
				rotate=0.0
			]{\bfseries $\vec{p}_\mathrm{ho}$};

			\addplot3[
				only marks,
				mark=*,
				mark size=1pt,
			] coordinates{(-0.4464148000467676, -0.2192111234588623, 0.7037999563464113)};
			\draw (axis cs:-0.44650080267280783, -0.3075765280398828, 0.7231554130968114) node[
				scale=1.0,
				anchor=base west,
				text=black,
				rotate=0.0
			]{\bfseries $\vec{p}_\mathrm{f}$};

		\end{axis}
	\end{tikzpicture}
	\caption{Experiment 3: Example handover trajectory from the handover dataset. The robot is visualized for the start, handover, and end configuration of the SafeFlowMPC trajectory. The human is represented by the red box at the handover location grabbing the green object from the robot\textquotesingle s end-effector.}
	\label{fig:scen_handover}
\end{figure}

While the objective for the point-to-point motions in Sections~\ref{sec:global_local_planning} and~\ref{sec:global_local_replanning} is easily encoded as a numerical function, i.e., minimizing the distance to the goal pose, this is not the case for many other objectives. Especially in human-robot interactions, it is difficult to find such functions. SafeFlowMPC allows using models learned from demonstrated trajectories instead of optimizing a scenario-specific reward function. We demonstrate this with a dynamic human-robot object handover, where a robot hands an object to a human as the human passes the robot. The human-human handover dataset from~\cite{kimLearningbasedDynamicRobottoHuman2025} is used for this task. The human hand trajectories of the giver are converted to the joint space of the robot using a trajectory optimization problem. The dataset contains 900 trajectories for training and 100 for testing of which 67 where successfully translated to the joint space. An example trajectory is depicted in the environment in~\cref{fig:scen_handover}. In this example, the human hand moves from the initial position $\vec{p}_{\mathrm{h, 0}}$ to the handover location $\vec{p}_{\mathrm{h, ho}}$ to grab the object and then continues to the final position $\vec{p}_{\mathrm{h, f}}$. The corresponding robot end-effector poses are $\vec{p}_{\mathrm{0}}$, $\vec{p}_{\mathrm{ho}}$, and $\vec{p}_{\mathrm{f}}$. At the handover location, the human hand is above the robot\textquotesingle s end-effector, as the object extends in the $z$-direction.
Safety is introduced by modeling the human body as an obstacle. This does not include the arm of the human as it needs to grab the object and come in close proximity to the end effector. As~\cref{eq:robot_h} limits the joint jerk of the robot, the resulting trajectory is smooth, which is important for the perceived safety for the human~\cite{ortenziObjectHandoversReview2021, zacharakiSafetyBoundsHuman2020}. Furthermore, the constraint
\begin{equation}
	\label{eq:handover_terminal}
	\vec{f}_{\mathrm{fk}}(\vec{q}(t_{0} + T)) \in \mathcal{S}_{T, \mathrm{ho}}
\end{equation}
is added to~\cref{eq:robot_h} as a terminal constraint to ensure that the robot returns the end-effector to a safe set when planning fails. As the human always passes the robot on the negative $y$-direction (see~\cref{fig:scen_handover}), the safe terminal set $\mathcal{S}_{T, \mathrm{ho}}$ constraints the end-effector to limit the $y$-position to remain close to the robot such that it cannot interfere with the human.
SafeFlowMPC computes the robot joint trajectory and the binary gripper state to learn the release timing of the object.
No guiding function $J_{\mathrm{guide}}(\vec{q}(t))$ is used here.

The comparison in~\cref{tab:exp2_comp} on the test dataset shows that SafeFlowMPC computes trajectories that are closest to the demonstrations while remaining safe according to~\cref{eq:handover_terminal} and real-time capable. Note that the average distance to the demonstrations $d_{\mathrm{demo}}$ assumes no reaction of the human to the differing robot motions, which is not applicable to real-world object handovers. Therefore, several real-world handovers are shown in the supplementary video, where the human hand is tracked with an OptiTrack system\footnote{OptiTrack https://www.optitrack.com/}. They show that SafeFlowMPC succeeds to hand over an object in the real world by learning from human demonstrations.

\begin{table}
	\caption{Experiment 3: Results of different planners in terms of average planning time $T_{\mathrm{plan}}$, terminal constraint violations $c_{T}>0$, and average position distance to the demonstrated trajectories $d_{\mathrm{demo}}$}\label{tab:exp2_comp}
	\centering
	\begin{tabular}[c]{lccc}
		Method             & $T_{\mathrm{plan}}$ / \si{\milli\second} & $c_{T}$ / \si{\meter} & $d_{\mathrm{demo}}$ / \si{\meter} \\
		\hline
		BC                 & $1 \pm 1$                                & $0.0$                 & $0.28$                            \\
		FM                 & $31 \pm 2$                               & $0.05$                & $0.21$                            \\
		SafeFlowMPC (Ours) & $63 \pm 6$                               & $0.0$                 & $0.12$                            \\
		\hline
	\end{tabular}
\end{table}

\section{Limitations}
Despite its strengths, SafeFlowMPC has notable limitations. It assumes the availability of a suitable dataset and requires a finetuning step to achieve optimal . Additionally, the method is restricted to constraints expressible as differentiable functions over the planning horizon. The design of safe terminal constraints also remains non-trivial, particularly for complex kinematics.
Furthermore, ensuring safety in interactive scenarios, such as the presented dynamic object handover in~\cref{sec:handover}, is generally difficult. In this work, we simplified the safety definition by considering the collisions with the current position of the human. However, the human is moving around, which is difficult to predict, requiring a trade-off between safety and performance~\cite{ortenziObjectHandoversReview2021}. A possible approach for future work is the prediction of the trajectory distribution of other actors~\cite{hewingSimulationTrajectoryPrediction2020} in combination with safety constraints~\cite{ortenziObjectHandoversReview2021}. This is out of the scope of this work as the dataset~\cite{kimLearningbasedDynamicRobottoHuman2025} contains very noisy data of the human movements, which is unsuitable for such models.
The object handover further simplifies the interactions within human-robot handovers by only considering the human hand trajectory as input for SafeFlowMPC. Improvements in performance can be gained by considering other factors such as trust, gaze, or other actors in the scene~\cite{zacharakiSafetyBoundsHuman2020}, which requires more expressive datasets.

\section{Conclusions and Future Work}
\label{sec:conclusions}
This work introduced SafeFlowMPC, a novel method for online trajectory planning that integrates safety manifolds with conditional flow-matching. The approach optimizes receding-horizon trajectories through an iterative process: a flow matching network generates desired motions, while an online optimizer enforces safety constraints. By carefully designing terminal constraints, SafeFlowMPC guarantees safety throughout operation.
We demonstrated the effectiveness of SafeFlowMPC in three challenging scenarios—static and dynamic object grasping with obstacles, and human-robot object handover—where it outperformed the baseline methods. Successful deployment on a real-world KUKA 7-DoF manipulator further validated the method’s real-time feasibility and practical applicability.
Future research will focus on extending SafeFlowMPC with learning-based safety filters~\cite{wabersichDataDrivenSafetyFilters2023} and terminal constraints, aiming to broaden its applicability and robustness in dynamic, real-world settings.



\bibliographystyle{IEEEtran}
\bibliography{IEEEabrv,safeflowmatching}  

\end{document}

%% file: tex_packages.tex
\usepackage[colorlinks,draft=false,allcolors=blue]{hyperref}
\usepackage{moreverb,url}
\usepackage{pgfplots}
\usepackage{pgfplotstable}
\usepgfplotslibrary{groupplots}
\usepackage{amsmath}
\usepackage{amsfonts}
\usepackage{amsthm}
\usepackage{import}
\usepackage{siunitx}
\usepackage{cleveref}
\usepackage{algorithm}
\usepackage{array}
\usepackage{algpseudocode}

%% file: tex_settings.tex
\crefformat{figure}{Fig.~#2#1#3}
\Crefformat{figure}{Figure~#2#1#3}
\crefformat{equation}{(#2#1#3)}
\crefformat{section}{Section~#2#1#3}
\crefformat{table}{Table~#2#1#3}
\crefformat{algorithm}{Algorithm~#2#1#3}
\crefformat{appendix}{Appendix~#2#1#3}
\crefformat{assumption}{Assumption~#2#1#3}
\crefformat{theorem}{Theorem~#2#1#3}

\DeclareRobustCommand{\vec}[1]{
	\ifthenelse{\equal{#1}{\omega} \OR \equal{#1}{\varphi} \OR \equal{#1}{\alpha} \OR \equal{#1}{\beta} \OR \equal{#1}{\chi} \OR \equal{#1}{\delta} \OR \equal{#1}{\varepsilon} \OR \equal{#1}{\phi} \OR \equal{#1}{\epsilon} \OR \equal{#1}{\gamma} \OR \equal{#1}{\eta} \OR \equal{#1}{\iota} \OR \equal{#1}{\kappa} \OR \equal{#1}{\lambda} \OR \equal{#1}{\mu} \OR \equal{#1}{\nu} \OR \equal{#1}{\pi} \OR \equal{#1}{\theta} \OR \equal{#1}{\vartheta} \OR \equal{#1}{\rho} \OR \equal{#1}{\sigma} \OR \equal{#1}{\varsigma} \OR \equal{#1}{\tau} \OR \equal{#1}{\upsilon} \OR \equal{#1}{\xi} \OR \equal{#1}{\psi} \OR \equal{#1}{\zeta}}{
		\boldsymbol{#1}
	}{
		\mathbf{#1}
	}
}

\usetikzlibrary {3d}
\usetikzlibrary{decorations.markings}
\usepgfplotslibrary{fillbetween}
\usepgfplotslibrary{patchplots}
\usetikzlibrary {arrows.meta}

\pgfplotsset{every axis/.append style={
			font=\footnotesize,
			line join=round,
			legend style={/tikz/every even column/.append style={column sep=0.2cm}}
		}}

\DeclareRobustCommand{\qs}[1]{\textsinglequote s}

\renewcommand{\dddot}[1]{%
	{\mathop{\kern\z@#1}\limits^{\makebox[0pt][c]{\vbox to-1.4\ex@{\kern-\tw@\ex@
						\hbox{\normalfont...}\vss}}}}}
\newcommand{\dddotsuper}[1]{%
	{\mathop{\kern\z@#1}\limits^{\makebox[0pt][c]{\vbox to-1.6\ex@{\kern-\tw@\ex@
						\hbox{\scriptsize...}\vss}}}}}
\newcommand{\dddotalt}[1]{%
	{\mathop{\kern\z@#1}\limits^{\makebox[0pt][c]{\vbox to-2.2\ex@{\kern-\tw@\ex@
						\hbox{\normalfont\scaleddot\kern-0.5pt\scaleddot\kern-0.5pt\scaleddot}\vss}}}}}
\makeatother

%% file: method_overview.pdf_tex
\begingroup%
  \makeatletter%
  \providecommand\color[2][]{%
    \errmessage{(Inkscape) Color is used for the text in Inkscape, but the package 'color.sty' is not loaded}%
    \renewcommand\color[2][]{}%
  }%
  \providecommand\transparent[1]{%
    \errmessage{(Inkscape) Transparency is used (non-zero) for the text in Inkscape, but the package 'transparent.sty' is not loaded}%
    \renewcommand\transparent[1]{}%
  }%
  \providecommand\rotatebox[2]{#2}%
  \newcommand*\fsize{\dimexpr\f@size pt\relax}%
  \newcommand*\lineheight[1]{\fontsize{\fsize}{#1\fsize}\selectfont}%
  \ifx\svgwidth\undefined%
    \setlength{\unitlength}{135.13558816bp}%
    \ifx\svgscale\undefined%
      \relax%
    \else%
      \setlength{\unitlength}{\unitlength * \real{\svgscale}}%
    \fi%
  \else%
    \setlength{\unitlength}{\svgwidth}%
  \fi%
  \global\let\svgwidth\undefined%
  \global\let\svgscale\undefined%
  \makeatother%
  \begin{picture}(1,0.6799063)%
    \lineheight{1}%
    \setlength\tabcolsep{0pt}%
    \put(0,0){\includegraphics[width=\unitlength,page=1]{method_overview.pdf}}%
  \end{picture}%
\endgroup%

%% file: planning_overview_robot.pdf_tex
\begingroup%
  \makeatletter%
  \providecommand\color[2][]{%
    \errmessage{(Inkscape) Color is used for the text in Inkscape, but the package 'color.sty' is not loaded}%
    \renewcommand\color[2][]{}%
  }%
  \providecommand\transparent[1]{%
    \errmessage{(Inkscape) Transparency is used (non-zero) for the text in Inkscape, but the package 'transparent.sty' is not loaded}%
    \renewcommand\transparent[1]{}%
  }%
  \providecommand\rotatebox[2]{#2}%
  \newcommand*\fsize{\dimexpr\f@size pt\relax}%
  \newcommand*\lineheight[1]{\fontsize{\fsize}{#1\fsize}\selectfont}%
  \ifx\svgwidth\undefined%
    \setlength{\unitlength}{178.79811084bp}%
    \ifx\svgscale\undefined%
      \relax%
    \else%
      \setlength{\unitlength}{\unitlength * \real{\svgscale}}%
    \fi%
  \else%
    \setlength{\unitlength}{\svgwidth}%
  \fi%
  \global\let\svgwidth\undefined%
  \global\let\svgscale\undefined%
  \makeatother%
  \begin{picture}(1,0.857965)%
    \lineheight{1}%
    \setlength\tabcolsep{0pt}%
    \put(0,0){\includegraphics[width=\unitlength,page=1]{planning_overview_robot.pdf}}%
  \end{picture}%
\endgroup%